\documentclass{article}

\usepackage[nonatbib,final]{neurips_2020}

\usepackage{pgfplots}
\pgfplotsset{compat=newest}
 \pgfplotsset{width=7cm,compat=1.15}
 \usepgfplotslibrary{patchplots}
\usepackage{tikz}
\usetikzlibrary{arrows.meta,shapes}
\usepackage{subcaption}
\usepackage{amsmath}
\usepackage[utf8]{inputenc} 
\usepackage[T1]{fontenc}    
\usepackage{hyperref}       
\usepackage{url}            
\usepackage{booktabs}       
\usepackage{amsfonts}       
\usepackage{nicefrac}       
\usepackage{microtype}      

\usepackage{amsthm}
\usepackage{thmtools,thm-restate}
\usepackage{algorithm}
\usepackage{algpseudocode}

\usepackage{appendix}
\usepackage{amssymb}

\usepackage{empheq}
\usepackage[mathscr]{euscript}
\usepackage{stmaryrd}
\let\origtriangleright\triangleright
\usepackage{mathabx}
\let\triangleright\origtriangleright
\usepackage{dsfont}
\usepackage{graphicx}
\usepackage{multirow}
\usepackage[normalem]{ulem}
\usepackage{tikz}
\usetikzlibrary{3d}
\usetikzlibrary{decorations.pathreplacing,angles,quotes,patterns}
\usepackage{pgfplots}
\usepackage{braket}
\usepackage[bb=boondox]{mathalfa}
\newcommand{\appendixref}[1]{{\autoref*{#1}}}
\usepackage{multirow}

\colorlet{Mycolor1}{green!30!black!70!}

\newcommand{\effd}{d}

\newcommand{\ubact}{\mathtt{ub\_used}}

\newcommand{\calC}{\mathcal{C}}

\newcommand{\calI}{\mathcal{I}}
\newcommand{\calJ}{\mathcal{J}}
\newcommand{\calL}{\mathcal{L}}

\newcommand{\calU}{\mathcal{U}}
\newcommand{\calV}{\mathcal{V}}

\newcommand{\bbR}{\mathbb{R}}

\newcommand{\bfzero}{\mathbf{0}}
\newcommand{\bfone}{\mathbf{1}}
\newcommand{\Conv}{\operatorname{Conv}}

\newcommand{\Proj}{\operatorname{Proj}}

\newcommand\defeq{\mathrel{\overset{\makebox[0pt]{\mbox{\normalfont\tiny\sffamily def}}}{=}}}

\newcommand{\tx}{\tilde{x}}

\newcommand{\abs}[1]{\left|#1\right|}                    
\newcommand{\rbra}[1]{\left(#1\right)}
\newcommand{\sidx}[1]{\left\llbracket     #1 \right\rrbracket}

\DeclareMathOperator{\conv}{conv}

\newcommand{\Elide}{\texttt{Elide}}

\newcommand{\FastLin}{\texttt{Fast-Lin}}
\newcommand{\DeepPoly}{\texttt{DeepPoly}}

\newcommand{\RefineZono}{\texttt{RefineZono}}
\newcommand{\kPoly}{\texttt{kPoly}}
\newcommand{\FastCtV}{\texttt{FastC2V}}
\newcommand{\OptCtV}{\texttt{OptC2V}}

\newtheorem{corollary}{Corollary}
\newtheorem{definition}{Definition}
\newtheorem{example}{Example}
\newtheorem{lemma}{Lemma}

\newtheorem{observation}{Observation}

\newtheoremstyle{named}{}{}{\itshape}{}{\bfseries}{.}{.5em}{\thmnote{#3}}
\theoremstyle{named}
\newtheorem*{namedexample}{Example}

\algnewcommand{\Inputs}[1]{%
  \State \textbf{Inputs:}
  \Statex \hspace*{\algorithmicindent}\parbox[t]{.85\linewidth}{\raggedright #1}
}
\algnewcommand{\Outputs}[1]{%
  \State \textbf{Outputs:}
  \Statex \hspace*{\algorithmicindent}\parbox[t]{.85\linewidth}{\raggedright #1}
}
\algnewcommand{\Parameters}[1]{%
  \State \textbf{Parameters:}
  \Statex \hspace*{\algorithmicindent}\parbox[t]{.85\linewidth}{\raggedright #1}
}
\algnewcommand{\Initialize}[1]{%
  \State \textbf{Initialize:}
  \Statex \hspace*{\algorithmicindent}\parbox[t]{.85\linewidth}{\raggedright #1}
}

\title{The Convex Relaxation Barrier, Revisited:\\ Tightened Single-Neuron Relaxations for Neural Network Verification}

\author{%
  Christian Tjandraatmadja \\
  Google Research \\
  \texttt{ctjandra@google.com} \\
  \And
  Ross Anderson \\
  Google Research \\
  \texttt{rander@google.com} \\
  \And
  Joey Huchette \\
  Rice University \\
  \texttt{joehuchette@rice.edu} \\
  \And
  Will Ma \\
  Columbia University \\
  \texttt{wm2428@gsb.columbia.edu} \\
  \And
  Krunal Patel \\
  Polytechnique Montréal\thanks{This work was completed while this author was at Google Research.} \\
  \texttt{krunal.patel@polymtl.ca} \\
  \And
  Juan Pablo Vielma \\
  Google Research \\
  \texttt{jvielma@google.com}
}

\begin{document}

\maketitle
\begin{abstract}
We improve the effectiveness of propagation- and linear-optimization-based neural network verification algorithms with a new tightened convex relaxation for ReLU neurons. Unlike previous single-neuron relaxations which focus only on the univariate input space of the ReLU, our method considers the multivariate input space of the affine pre-activation function preceding the ReLU. Using results from submodularity and convex geometry, we derive an explicit description of the tightest possible convex relaxation when this multivariate input is over a box domain. We show that our convex relaxation is significantly stronger than the commonly used univariate-input relaxation which has been proposed as a natural convex relaxation barrier for verification. While our description of the relaxation may require an exponential number of inequalities, we show that they can be separated in linear time and hence can be efficiently incorporated into optimization algorithms on an as-needed basis. Based on this novel relaxation, we design two polynomial-time algorithms for neural network verification: a linear-programming-based algorithm that leverages the full power of our relaxation, and a fast propagation algorithm that generalizes existing approaches. In both cases, we show that for a modest increase in computational effort, our strengthened relaxation enables us to verify a significantly larger number of instances compared to similar algorithms.
\end{abstract}

\section{Introduction}


A fundamental problem in deep neural networks is to \emph{verify} or \emph{certify} that a trained network is \emph{robust}, i.e. not susceptible to adversarial attacks \cite{Carlini:2016,Papernot:2016,Szegedy:2013}. Current approaches for neural network verification can be divided into \emph{exact} (\emph{complete}) methods and \emph{relaxed} (\emph{incomplete}) methods. Exact verifiers are often based on mixed integer programming (MIP) or more generally branch-and-bound \cite{anderson2020strong,Strong-mixed-integer-programming-formulations-for-trained-CONF,botoevaefficient,bunel2020branch,cheng2017maximum,dutta2018output,fischetti2017deep,lomuscio2017approach,lu2020neural,Roessig2019,Tjeng:2017,Xiao:2018} or satisfiability modulo theories (SMT) \cite{ehlers2017formal,huang2017safety,Katz:2017,narodytska2018verifying,scheibler2015towards} and, per their name, exactly solve the problem, with no false negatives or false positives. However, exact verifiers are typically based on solving NP-hard optimization problems \cite{Katz:2017} which can significantly limit their scalability. In contrast, relaxed verifiers are often based on polynomially-solvable optimization problems such as convex optimization or linear programming (LP) \cite{anderson2020tightened,bhojanapalli2020efficient,dvijotham2018dual,liu2019training,lyu2019fastened,raghunathan2018semidefinite,salman2019convex,singh2019beyond,xiang2018output,zhu2020improving}, which in turn lend themselves to faster \emph{propagation-based} methods where bounds are computed by a series of variable substitutions in a backwards pass through the network~\cite{singh2019abstract,weng2018towards,wong2018provable,Wong:2018,zhang2018efficient}.  Unfortunately, relaxed verifiers achieve this speed and scalability by trading off effectiveness (i.e. increased false negative rates), possibly failing to certify robustness when robustness is, in fact, present. As might be expected, the success of relaxed methods hinges on their tightness, or how closely they approximate the object which they are relaxing.

As producing the tightest possible relaxation for an entire neural network is no easier than the original verification problem, most relaxation approaches turn their attention instead to simpler substructures, such as individual neurons. For example, the commonly used \emph{$\Delta$-relaxation}\footnote{Sometimes also called the \emph{triangle relaxation}~\cite{liu2019algorithms,singh2019beyond}.}\cite{ehlers2017formal} is simple and offers the tightest possible relaxation for the univariate ReLU function, and as a result is the foundation for many relaxed verification methods. Recently, Salman et al. \cite{salman2019convex} characterized the \emph{convex relaxation barrier}, showing that the effectiveness of all existing propagation-based fast verifiers is fundamentally limited by the tightness of this $\Delta$-relaxation. Unfortunately, they show computationally that this convex barrier can be a severe limitation on the effectiveness of relaxed verifiers based upon it. While the convex relaxation barrier can be bypassed in various ways (e.g.\ considering relaxations for multiple neurons~\cite{singh2019beyond}), as noted in \cite[Appendix A]{salman2019convex} all existing approaches that achieve this do so by trading off clarity and speed.

In this paper we improve the effectiveness of propagation- and LP-based relaxed verifiers  with a new tightened convex relaxation for ReLU neurons. Unlike the $\Delta$-relaxation which focuses only on the univariate input space of the ReLU, our relaxation considers the multivariate input space of the affine pre-activation function preceding the ReLU. By doing this, we are able to bypass the convex barrier from \cite{salman2019convex} while remaining in the realm of single-neuron relaxations that can be utilized by fast propagation- and LP-based verifiers.

More specifically, our contributions are as follows.
\begin{enumerate}
\item Using results from submodularity and convex geometry, we derive an explicit linear inequality description for the tightest possible convex relaxation of a single neuron, where, in the spirit of \cite{anderson2020strong,Strong-mixed-integer-programming-formulations-for-trained-CONF}, we take this to encompass the ReLU activation function, the affine pre-activation function preceding it, and known bounds on each input to this affine function. We show that this new convex relaxation is significantly stronger than the $\Delta$-relaxation, and hence bypasses the convex barrier from \cite{salman2019convex} without the need to consider multi-neuron interactions as in, e.g. \cite{singh2019beyond}.
\item We show that this description, while requiring an exponential number of inequalities in the worst case, admits an efficient separation routine. In particular, we present a linear time algorithm that, given a point, either asserts that this point lies within the relaxation, or returns an inequality that is not satisfied by this point. Using this routine, we develop two verification algorithms that incorporate our tighter inequalities into the relaxation.
\begin{enumerate}
    \item \OptCtV{}: We develop a polynomial-time LP-based algorithm that harnesses the full power of our new relaxation.
    \item \FastCtV{}: We develop a fast propagation-based algorithm that generalizes existing approaches (e.g.\  \FastLin{}~\cite{weng2018towards} and \DeepPoly{}~\cite{singh2019abstract}) by
    dynamically adapting the relaxation using our new inequalities.
\end{enumerate}
    \item Computational experiments on verification problems using networks from the ERAN dataset~\cite{eran_benchmark} demonstrate that leveraging these inequalities yields a substantial improvement in verification capability. In particular, our fast propagation-based algorithm surpasses the strongest possible algorithm restricted by the convex barrier (i.e.\ optimizing over the $\Delta$-relaxation at every neuron). We also show that our methods are competitive with more expensive state-of-the-art methods such as \RefineZono{}~\cite{singh2019boosting} and \kPoly{}~\cite{singh2019beyond}, certifying more images than them in several cases.
\end{enumerate}

\section{Verification via mathematical optimization}
Consider a neural network $f : \bbR^{m} \to \bbR^r$ described in terms of $N$ neurons in a linear order.\footnote{This allows us to consider feedforward networks, including those that skip layers  (e.g.\ see \cite{salman2019convex,zhu2020improving}).}
 The first $m$ neurons are the \emph{input neurons}, while the remaining \emph{intermediate neurons} are indexed by $i=m+1,\ldots,N$. Given some input $x \in \bbR^m$, the relationship $f(x) = y$ can be described as
\begin{subequations} 
\label{eqn:topo-network-equations}
\begin{alignat}{4}
    x_i &= z_i \quad &\forall& i = 1,\ldots,m \quad &\text{(the inputs)} \label{eqn:topo-network-equations-1} \\
    \hat{z}_i &= \sum\nolimits_{j=1}^{i-1} w_{i,j}z_j + b_i \quad &\forall& i = m+1, \ldots, N \quad &\text{(the pre-activation value)} 
    \label{eqn:topo-network-equations-2} \\
    z_i &= \sigma(\hat{z}_i) \quad &\forall& i = m+1, \ldots, N \quad &\text{(the post-activation value)}
    \label{eqn:topo-network-equations-3} \\
    y_i &= \sum\nolimits_{j=1}^N w_{i,j}z_j + b_{i} \quad &\forall& i = N+1,\ldots,N+r \quad &\text{(the outputs)}. \label{eqn:topo-network-equations-4}
\end{alignat}
\end{subequations}
Here the constants $w$ and $b$ are the weights and biases, respectively, learned during training, while $\sigma(v) \defeq \max\{0,v\}$ is the ReLU activation function. Appropriately, for each neuron $i$ we dub the variable $\hat{z}_i$ the \emph{pre-activation variable} and $z_i$ the \emph{post-activation variable}.

Given a trained network (i.e. fixed architecture, weights, and biases), we study a \emph{verification problem} of the following form: given constant $c \in \bbR^r$, polyhedron $X \subseteq \bbR^m$, $\beta \in \bbR$, and
\begin{equation}\label{verification_problem}
   \gamma(c,X) \defeq \max\nolimits_{x \in X} c \cdot f(x) \equiv \max\nolimits_{x,y,\hat{z},z}\Set{c\cdot y |x\in X,\quad \eqref{eqn:topo-network-equations} },
\end{equation}
does $\gamma(c,X) \leq \beta$? Unfortunately, this problem is NP-hard~\cite{Katz:2017}. Moreover, one is typically not content with solving just one problem of this form, but would like to query for many reasonable choices of $c$ and $X$ to be convinced that the network is robust to adversarial perturbations.

A promising approach to approximately solving the verification problem is to replace the intractable optimization problem defining $\gamma$ in \eqref{verification_problem} with a tractable \emph{relaxation}. In particular, we aim to identify a tractable optimization problem whose optimal objective value $\gamma_R(c,X)$ satisfies $\gamma(c,X) \leq \gamma_R(c,X)$, for all parameters $c$ and $X$ of interest. Then, if $\gamma_R(c,X) \leq \beta$, we have answered the verification problem in the affirmative. However, note that it may well be the case that, by relaxing the problem, we may fail to verify a network that is, in fact, verifiable (i.e. $\gamma(c,X) \leq \beta < \gamma_R(c,X)$). Therefore, \emph{the strength of our relaxation is crucial for reducing the false negative rate of our verification method.}

\subsection{The \texorpdfstring{$\Delta$}{triangle}-relaxation and its convex relaxation barrier}
Salman et al.~\cite{salman2019convex} note that many relaxation approaches for ReLU networks are based on the single-activation-function set $A^{i} \defeq \{(\hat{z}_i,z_i) \in \bbR^2\ |\ \hat{L}_i \leq \hat{z}_i \leq \hat{U}_i, \: z_i = \sigma_j(\hat{z}_i)\}$, where the pre-activation bounds $\hat{L}_i, \hat{U}_i\in \bbR$ are taken so that $\hat{L}_i \leq \hat{z}_i \leq \hat{U}_i$  for any point that satisfies $x\in X$ and \eqref{eqn:topo-network-equations}.
The $\Delta$-relaxation $C_{\Delta}^i \defeq \Conv(A^i)$ is optimal in the sense that it describes the convex hull of $A^i$, with three simple linear inequalities:
$ z_i\geq 0,\ z_i \geq \hat{z}_i,$\ and $z_i\leq \frac{\hat{U}_i}{\hat{U}_i-\hat{L}_i}(\hat{z}_i - \hat{L}_i)$.

The simplicity and small size of the $\Delta$-relaxation is appealing, as it leads to the relaxation
  \begin{equation}\label{verification_problem_vtwo_triangle_relaxation}
      \gamma_{\Delta}(c,X)\defeq\max_{x,y,\hat{z},z} \Set{
        c\cdot y |
        x\in X, \quad
        \eqref{eqn:topo-network-equations-1},\eqref{eqn:topo-network-equations-2},\eqref{eqn:topo-network-equations-4}, \quad
        (\hat{z}_i,z_i)\in C^{i}_{\Delta}\: \forall i = m+1,\ldots,N
      }.
 \end{equation}
 This is a small\footnote{Here, ``small'' means the number of variables and constraints is $\mathcal{O}(\# \text{ of neurons})$.} Linear Programming (LP) problem than is theoretically tractable and relatively easy to solve in practice. Moreover, a plethora of fast propagation-based algorithms~\cite{singh2018fast, singh2019abstract, wang2018efficient, weng2018towards, wong2018provable, zhang2018efficient} center on an approach that can be interpreted as further relaxing $\gamma_\Delta$, where inequalities describing the sets $C^i_\Delta$ are judiciously dropped from the description in such a way that this LP becomes much easier to solve. Unfortunately, Salman et al.~\cite{salman2019convex} observe that the quality of the verification bounds obtained through the $\Delta$-relaxation are intrinsically limited; a phenomenon they call the \emph{convex relaxation barrier}. Nonetheless, this LP, along with faster propagation algorithms that utilize the inequalities defining $C^i_\Delta$, have been frequently  applied to the verification task, often with substantial success.

\subsection{Our approach: Eliding pre-activation variables}
In this paper, we show that we can significantly improve over the accuracy of $\Delta$-relaxation verifiers with only a minimal trade-off in simplicity and speed. The key for this result is the observation that pre-activation variables are a ``devil in disguise'' in the context of convex relaxations. For a neuron $i$, the pre-activation variable $\hat{z}_i$ and the post-activation variable $z_i$ form the minimal set of variables needed to capture (and relax) the nonlinearity introduced by the ReLU. However, this approach ignores the inputs to the pre-activation variable $\hat{z}_i$, i.e. the preceding post-activation variables $z_{1:i-1} \defeq (z_1,\ldots,z_{i-1})$.


Our approach captures these relationships by instead turning our attention to the $i$-dimensional set\footnote{The \emph{effective} dimension of this set can be much smaller if $w_{i,\cdot}$ is sparse. This is the case with a feedforward network, where the number of nonzeros is (at most) the number of neurons in the preceding layer.} $S^{i} \defeq \Set{z \in \bbR^i | L \leq z_{1:i-1} \leq U, \quad z_i = \sigma\left(\sum_{j=1}^{i-1} w_{i,j}z_j + b_i\right)}$, where the post-activation bounds $L,U \in \bbR^{i-1}$ are such that $L_j \leq z_j \leq U_j$ for each point satisfying $x \in X$ and \eqref{eqn:topo-network-equations}. Note that no pre-activation variables appear in this description; we elide them completely, substituting the affine function describing them inside of the activation function.

\tikzstyle{yzx} = [
  x={(.9625cm, .9625cm)},
  y={(2.5cm, 0cm)},
  z={(0cm, 2.5cm)},
]

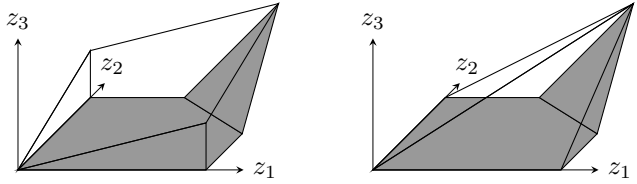
\begin{figure}[htpb]
    \centering
    \begin{tikzpicture}[yzx]
        \draw [->, >=stealth, line width=0.25] (0,0,0) -- (1.2,0,0);
        \draw [->, >=stealth, line width=0.25] (0,0,0) -- (0,1.2,0);
        \draw [->, >=stealth, line width=0.25] (0,0,0) -- (0,0,0.7);
        \node[above right] at (1.2,-.075,0) {$z_2$};
        \node[right] at (0,1.2,0) {$z_1$};
        \node[above] at (0,0,.7) {$z_3$};
        \coordinate (LL) at (0,0,0);
        \coordinate (UL) at (1,0,0);
        \coordinate (LU) at (0,1,0);
        \coordinate (UU) at (1,1,0.5);
        \coordinate (UM) at (1,0.5,0);
        \coordinate (MU) at (0.5,1,0);
        \coordinate (E1) at (1,0,1/4);
        \coordinate (E2) at (0,1,1/4);

        \draw [fill=gray!80] (LL) -- (UL) -- (UM) -- (MU) -- (LU) -- cycle;
        \draw [fill=gray!80] (UU) -- (UM) -- (MU) -- cycle;
        \draw (LL) -- (E2) -- (UU) -- (E1) -- cycle;
        \draw (LL) -- (UL) -- (E1) -- cycle;
        \draw (LL) -- (LU) -- (E2) -- cycle;

    \end{tikzpicture} \hspace{2em}
    \begin{tikzpicture}[yzx]
        \draw [->, >=stealth, line width=0.25] (0,0,0) -- (1.2,0,0);
        \draw [->, >=stealth, line width=0.25] (0,0,0) -- (0,1.2,0);
        \draw [->, >=stealth, line width=0.25] (0,0,0) -- (0,0,0.7);
        \node[above right] at (1.2,-.075,0) {$z_2$};
        \node[right] at (0,1.2,0) {$z_1$};
        \node[above] at (0,0,.7) {$z_3$};
        \coordinate (LL) at (0,0,0);
        \coordinate (UL) at (1,0,0);
        \coordinate (LU) at (0,1,0);
        \coordinate (UU) at (1,1,0.5);
        \coordinate (UM) at (1,0.5,0);
        \coordinate (MU) at (0.5,1,0);

        \draw [fill=gray!80] (LL) -- (UL) -- (UM) -- (MU) -- (LU) -- cycle;
        \draw [fill=gray!80] (UU) -- (UM) -- (MU) -- cycle;
        \draw (LL) -- (UL) -- (UU) -- cycle;
        \draw (LL) -- (LU) -- (UU) -- cycle;

    \end{tikzpicture}
    \caption{A simple neural network with $m=2$ dimensional input and one intermediate neuron ($N=3$). \textbf{(Left)} The feasible region for $\gamma_\Delta$, and \textbf{(Right)} The feasible region for $\gamma_\Elide$. The $x$, $y$, and $\hat{z}$ variables, which depend affinely on the others, are projected out.
    }
    \label{fig:relu}
\end{figure}

This immediately gives a single-neuron relaxation of the form
\begin{equation}\label{verification_problem_vtwo_oneneuron_relaxation}
      \gamma_{\Elide}(c,X) \defeq\max_{x,y,z}\Set{
        c\cdot y |
        x \in X,\quad
        \eqref{eqn:topo-network-equations-1},\eqref{eqn:topo-network-equations-4}, \quad
        z_{1:i} \in C^{i}_{\Elide}\: \forall i=m+1,\ldots,N
      },
 \end{equation}
 where $C^{i}_{\Elide}\defeq \Conv(S^{i})$ is the convex hull of $S^i$, as shown in Figure~\ref{fig:relu} (adapted from \cite{anderson2020strong}), which contrasts it with the convex barrier and $\Delta$-relaxation. We will show that, unsurprisingly, $C^i_\Elide$ will require exponentially many inequalities to describe in the worst case. However, \emph{we show that this need not be a barrier to incorporating this tighter relaxation into verification algorithms.}

\section{An exact convex relaxation for a single ReLU neuron}\label{exactconvexsec}

Let $w\in \bbR^n$, $b\in \bbR$, $f(x)\defeq w\cdot x+b$, and $L,U\in\bbR^n$ be such that $L< U$. For ease of exposition, we  rewrite the single-neuron set $S^i$ in the generic form \begin{equation}\label{Sdef:eq}
    S \defeq \Set{(x,y)\in [L,U]\times \bbR |  y = \sigma(f(x)) }.
\end{equation}
  Notationally, take $\llbracket n \rrbracket \defeq \{1,\ldots,n\}$, $\breve{L}_i \defeq \begin{cases} L_i & w_i \geq 0 \\ U_i & \text{o.w.} \end{cases}$ and $\breve{U}_i \defeq \begin{cases} U_i & w_i \geq 0 \\ L_i & \text{o.w.} \end{cases}$ for each $i\in\sidx{n}$, $\ell(I) \defeq \sum_{i \in I} w_i\breve{L}_i + \sum_{i \not\in I} w_i\breve{U}_i + b$, and
    \[\calJ\defeq\Set{
        (I,h) \in 2^{\llbracket n \rrbracket} \times \llbracket n \rrbracket |
                    \ell(I) \geq  0,\quad
            \ell(I\cup\{ h\}) < 0,\quad w_i\neq 0 \; \forall i\in I
    }.\]

Our main technical result uses results from submodularity and convex geometry \cite{ahmed2011maximizing, bach2013learning, o1971hyperplane,tawarmalani2013explicit} to give the following closed-form characterization of $\Conv(S)$. For a proof of \autoref{thm:convex-hull}, see \appendixref{theooneproof}.


\begin{restatable}{theorem}{generaltheo} \label{thm:convex-hull}
If $\ell(\sidx{n})\geq 0$, then $\Conv(S)= S = \Set{(x,y)\in [L,U]\times \bbR |  y = f(x) }$.
Alternatively, if $\ell(\emptyset)< 0$, then $\Conv(S) = S = [L,U]\times \set{0} $.
Otherwise, $\Conv(S)$ is equal to the set of all $(x,y)\in\bbR^n\times\bbR$ satisfying
\begin{subequations} \label{eqn:relu-conv-hull}
\begin{align}
y &\geq w\cdot x + b, \quad y \geq 0, \quad L \leq x \leq U \label{eqn:relu-conv-remaining-constrs} \\
y &\leq \sum\nolimits_{i \in I} w_i(x_i-\breve{L}_i) + \frac{\ell(I)}{\breve{U}_{ h}-\breve{L}_{ h}} (x_{ h}-\breve{L}_{ h}) \quad &\forall (I, h) \in \calJ. \label{eqn:cut-family-2}
\end{align}
\end{subequations}
Furthermore, if  $\effd\defeq\abs{\Set{i\in \sidx{n}| w_i\neq 0}}$, then $\effd \leq \abs{\calJ}\leq \lceil \frac{1}{2} \effd \rceil \binom{\effd}{\lceil \frac{1}{2} \effd \rceil}$ and for each of these inequalities (and each $d\in \sidx{n}$) there exist data that makes it hold at equality.

\end{restatable}

Note that this is the tightest possible relaxation when $x \in [L, U]$. Moreover, we observe that the relaxation offered by $\Conv(S)$ can be arbitrarily tighter than that derived from the $\Delta$-relaxation.

  \begin{restatable}{proposition}{arbitrarygap}\label{prop:arbitrary-gap}
    For any input dimension $n$, there exists a point $\tilde{x}\in \bbR^n$, and a problem instance given by the affine function $f$, the $\Delta$-relaxation $C_\Delta$, and the single neuron set $S$ such that $\left(\max_{y : (f(\tilde{x}),y) \in C_\Delta} y\right) - \left(\max_{y : (\tilde{x},y) \in \Conv(S)} y\right) = \Omega(n)$.
  \end{restatable}

Although the family of upper-bounding constraints \eqref{eqn:cut-family-2} may be exponentially large, the structure of the inequalities is remarkably simple. As a result, the \emph{separation problem} can be solved efficiently: given $(x,y)$, either verify that $(x,y) \in \Conv(S)$, or produce an inequality from the description \eqref{eqn:relu-conv-hull} which is violated at $(x,y)$. For instance, we can solve in $\mathcal{O}(n \log n)$ time  the optimization problem
\begin{equation}\label{separationproblem}
    \upsilon(x)\defeq\min\Set{ \sum\nolimits_{i \in I} w_i(x_i-\breve{L}_i) + \frac{\ell(I)}{\breve{U}_{ h}-\breve{L}_{ h}} (x_{ h}-\breve{L}_{ h})  | (I, h) \in \calJ },
\end{equation}
by sorting the indices with $w_i\neq 0$ in nondecreasing order of values $(x_i - \breve{L}_i)/(\breve{U}_i - \breve{L}_i)$, then adding them to $I$ in this order so long as $\ell(I) \geq 0$ (note that adding to $I$ can only decrease $\ell(I)$), and then letting $h$ be the index that triggered the stopping condition $\ell(I \cup \{h\}) < 0$. For more details, see the proof of \autoref{lemma:general_separation} in \appendixref{otherproofsfromthree}.



Then, to check if $(x,y) \in \Conv(S)$, we first check if the point satisfies \eqref{eqn:relu-conv-remaining-constrs}, which can be accomplished in $\mathcal{O}(n)$ time. If so, we compute $\upsilon(x)$ in $\mathcal{O}(n \log n)$ time.
If $y\leq \upsilon(x)$, then $(x,y) \in \Conv(S)$. Otherwise, an optimal solution to \eqref{separationproblem} yields an inequality from \eqref{eqn:cut-family-2} that is most violated at $(x,y)$. In addition, we can also solve \eqref{separationproblem} slightly faster.
\begin{restatable}{proposition}{separationlemma}\label{lemma:general_separation}
The optimization problem \eqref{separationproblem}
can be solved in $\mathcal{O}(n)$ time.

\end{restatable}
Together with the ellipsoid algorithm~\cite{grotschel2012geometric}, \autoref{lemma:general_separation} shows that the single-neuron relaxation $\gamma_\Elide$ can be efficiently solved (at least in a theoretical sense).
\begin{restatable}{corollary}{polylemma}\label{polylemma:ref}
If the weights $w$ and biases $b$ describing the neural network are rational, then the single-neuron relaxation \eqref{verification_problem_vtwo_oneneuron_relaxation}
 can be solved in polynomial time on the encoding sizes of $w$ and $b$.
\end{restatable}
For proofs of \autoref{prop:arbitrary-gap}, \autoref{lemma:general_separation} and \autoref{polylemma:ref}, see \appendixref{otherproofsfromthree}.

\paragraph{Connections with Anderson et al.~\cite{anderson2020strong,Strong-mixed-integer-programming-formulations-for-trained-CONF}}
Anderson et al.~\cite{anderson2020strong,Strong-mixed-integer-programming-formulations-for-trained-CONF} have previously presented a MIP formulation that exactly models the set $S$ in \eqref{Sdef:eq}. This formulation is \emph{ideal} so, in particular, its LP relaxation offers a \emph{lifted} LP formulation with one \emph{auxiliary variable} whose projection onto the \emph{original variables} $x$ and $y$ is exactly  $\Conv(S)$. Indeed, in Appendix~\ref{secondproof} we provide an alternative derivation for Theorem~\ref{thm:convex-hull} using the machinery presented in \cite{anderson2020strong}. This lifted LP can be used in lieu of our new formulation \eqref{eqn:relu-conv-hull}, though it offers no greater strength and requires an additional $N-m$ variables if applied for each neuron in the network. Moreover, it is not clear how to incorporate the lifted LP into propagation-based algorithms to be presented in the following section, which naturally work in the original variable space.

\section{A propagation-based algorithm}\label{propagationsec}
We now present a technique to use the new family of strong inequalities \eqref{eqn:cut-family-2} to generate strong post-activation bounds for a trained neural network. A step-by-step example of this method is available in Appendix~\ref{app:fastc2v_example}. To properly define the algorithm, we begin by restating a generic propagation-based bound generation framework under which various algorithms from the literature are special cases (partially or completely)~\cite{singh2018fast, singh2019abstract, wang2018efficient, weng2018towards, wong2018provable, zhang2018efficient}.

\subsection{A generic framework for computing post-activation bounds}\label{sec:generic_framework}

Consider a bounded input domain $X \subseteq \bbR^m$, along with a single output (i.e. $r=1$) to be maximized, which we name $\calC(z)=\sum_{i=1}^\eta c_iz_i + b$ for some $\eta\leq N$. In this section, our goal is produce efficient algorithms for producing valid upper bounds for $\calC$. First, let $z_i(x)$ denote the unique value of $z_i$ (post-activation variable $i$) implied by the equalities~(\ref{eqn:topo-network-equations-2}--\ref{eqn:topo-network-equations-3}) when we set $z_{1:m}=x$ for some $x \in X$.
Next, assume that for each intermediate neuron $i = m+1,\ldots,\eta$ we have affine functions of the form $\calL_i(z_{1:i-1})=\sum_{j=1}^{i-1}w^l_{ij}z_j+b^l_i$ and $\calU_i(z_{1:i-1})=\sum_{j=1}^{i-1}w^u_{ij}z_j+b^u_i$, such that
\begin{align} \label{eqn:affineValid}
\calL_i(z_{1:i-1}(x))\le z_i(x)\le\calU_i(z_{1:i-1}(x))\quad\forall x\in X,\quad i=1,\ldots,\eta.
\end{align}
We consider how to construct these functions in the next subsection.
Then, given these functions we can compute a bound on $\calC\rbra{z_{1:\eta}\rbra{x}}$ through the following optimization problem:
\begin{subequations} \label{eqn:relaxed-problem}
\begin{align}
B\rbra{\calC,\eta}\defeq \max_{z}\quad& \calC(z) \equiv \sum\nolimits_{i=1}^\eta c_i z_i + b \label{eqn:relaxedObj} \\
\text{s.t.}\quad& z_{1:m}\in X \\
& \calL_i(z_{1:i-1})\le z_i\le\calU_i(z_{1:i-1})\quad\forall i=m+1,\ldots,\eta. \label{eqn:relaxedConstr}
\end{align}
\end{subequations}
\begin{restatable}{proposition}{relaxequallp}
    The optimal value of \eqref{eqn:relaxed-problem} is no less than $\max_{x\in X} \calC\rbra{z_{1:\eta}\rbra{x}}$.
\end{restatable}


The optimal value $B\rbra{\calC,\eta}$ can be quickly computed through propagation methods without explicitly computing an optimal solution to~\eqref{eqn:relaxed-problem}~\cite{salman2019convex,zhang2018efficient}. Such methods perform a backward pass to sequentially eliminate (project out) the intermediate variables $z_N,\ldots,z_{m+1}$, which can be interpreted as applying Fourier-Motzkin elimination~\cite[Chapter 2.8]{Bertsimas:1997}. In a nutshell, for $i = \eta, \ldots, m+1$, the elimination step for variable $z_i$ uses its objective coefficient (which may be changing throughout the algorithm) to determine which one of the bounds from \eqref{eqn:relaxedConstr} will be binding at the optimal solution and replaces $z_i$ by the expression  $\calL_i(z_{1:i-1})$ or $\calU_i(z_{1:i-1})$ accordingly. The procedure ends with a smaller LP that only involves the input variables $z_{1:m}$ and can be quickly solved with an appropriate method. For instance, when $X$ is a box, as is common in verification problems, this final LP can be trivially solved by considering each variable individually. For more details, see Algorithm~\ref*{alg:backwards} in \appendixref{propagationalgappendix}.

\subsection{Selecting the bounding functions}\label{ssec:selecting-bounding-functions}
The framework described in the previous section required as input the family of bounding functions $\{\calL_i,\calU_i\}_{i=m+1}^\eta$. A typical approach to generate these will proceed sequentially, deriving the $i$-th pair of functions using \emph{scalar bounds} $\hat{L}_i,\hat{U}_i \in \bbR$
on the $i$-th  pre-activation variables $\hat{z}_{i}$, which by \eqref{eqn:topo-network-equations-2} is equal to $ \sum\nolimits_{j=1}^{i-1} w_{i,j}z_j + b_i$. Hence, these scalar bounds must satisfy
\begin{align} \label{eqn:scalarValid}
\hat{L}_i\le \sum\nolimits_{j=1}^{i-1} w_{i,j}z_j(x) + b_i\le \hat{U}_i\quad \forall x \in X.
\end{align}
These bounds can then be used as a basis to linearize the nonlinear equation
\begin{equation}\label{propagnonlineareq}
    z_i = \sigma\rbra{\sum\nolimits_{j=1}^{i-1} w_{i,j}z_j + b_i}
\end{equation}
 implied by (\ref{eqn:topo-network-equations-2}-\ref{eqn:topo-network-equations-3}).
If $\hat{U}_i \leq 0$ or $\hat{L}_i \geq 0$, then \eqref{propagnonlineareq} behaves linearly when \eqref{eqn:scalarValid} holds, and so we can let
$\calL_i(z_{1:i-1}) = \calU_i(z_{1:i-1}) = \sum\nolimits_{j=1}^{i-1} w_{i,j}z_j + b_i$ or $\calL_i(z_{1:i-1}) = \calU_i(z_{1:i-1}) = 0$, respectively. Otherwise, we can construct non-trivial bounds such as
\[
    \calL_i(z_{1:i-1}) = \frac{\hat{U}_i}{\hat{U}_i-\hat{L}_i}\rbra{\sum_{j=1}^{i-1} w_{i,j}z_j + b_i} \;\text{ and }\quad \calU_i(z_{1:i-1}) = \frac{\hat{U}_i}{\hat{U}_i-\hat{L}_i}\rbra{\sum_{j=1}^{i-1} w_{i,j}z_j + b_i-\hat{L}_i},
    \]
which can be derived from the  $\Delta$-relaxation: $\calU_i(z_{1:i-1})$ is the single upper-bounding inequality present on the left side of \autoref{fig:relu}, and $\calL_i(z_{1:i-1})$ is a shifted down version of this inequality.\footnote{Note that these functions satisfy \eqref{eqn:affineValid} only when $\hat{U}_i > 0$ and $\hat{L}_i < 0$.}
This pair is used by algorithms such as \FastLin{}~\cite{weng2018towards}, \texttt{DeepZ}~\cite{singh2018fast}, \texttt{Neurify}~\cite{wang2018efficient}, and that of Wong and Kolter~\cite{wong2018provable}. Algorithms such as \DeepPoly{}~\cite{singh2019abstract} and \texttt{CROWN-Ada}~\cite{zhang2018efficient} can be derived by selecting the same $\calU_i(z_{1:i-1})$ as above and $\calL_i(z_{1:i-1}) = 0$ if $|\hat{L}_i| \geq |\hat{U}_i|$ or $\calL_i(z_{1:i-1}) = \sum\nolimits_{j=1}^{i-1} w_{i,j}z_j + b_i$ otherwise (i.e.\ whichever yields the smallest area of the relaxation).
In the next subsection, we propose using~\eqref{eqn:cut-family-2} for $\calU_i(z_{1:i-1})$.

Scalar bounds satisfying \eqref{eqn:scalarValid} for the $i$-th pre-activation variable
can be computed by letting $\calC^{U,i}\rbra{z_{1:(i-1)}}=\sum\nolimits_{j=1}^{i-1} w_{i,j}z_j + b_{i}$ and then setting $\hat{L}_i=-B\rbra{\calC^{L,i},i-1}$ and $\hat{U}_i=B\rbra{\calC^{U,i},i-1}$. Therefore, to reach a final bound for $\eta = N$, we can iteratively compute $\hat{L}_i$ and $\hat{U}_i$ for $i = m+1,\ldots,N$ by solving~\eqref{eqn:relaxed-problem} each time, since each of these problems requires only affine bounding functions up to intermediate neuron $i-1$. See Algorithm~\ref{alg:full} in \appendixref{propagationalgappendix} for details.

\subsection{Our contribution: Tighter bounds by dynamically updating bounding functions}\label{sec:new_prop_algorithm}

\begin{figure}[t]
\centering
\begin{subfigure}[t]{.3\linewidth}
\centering
\includegraphics[width=\textwidth]{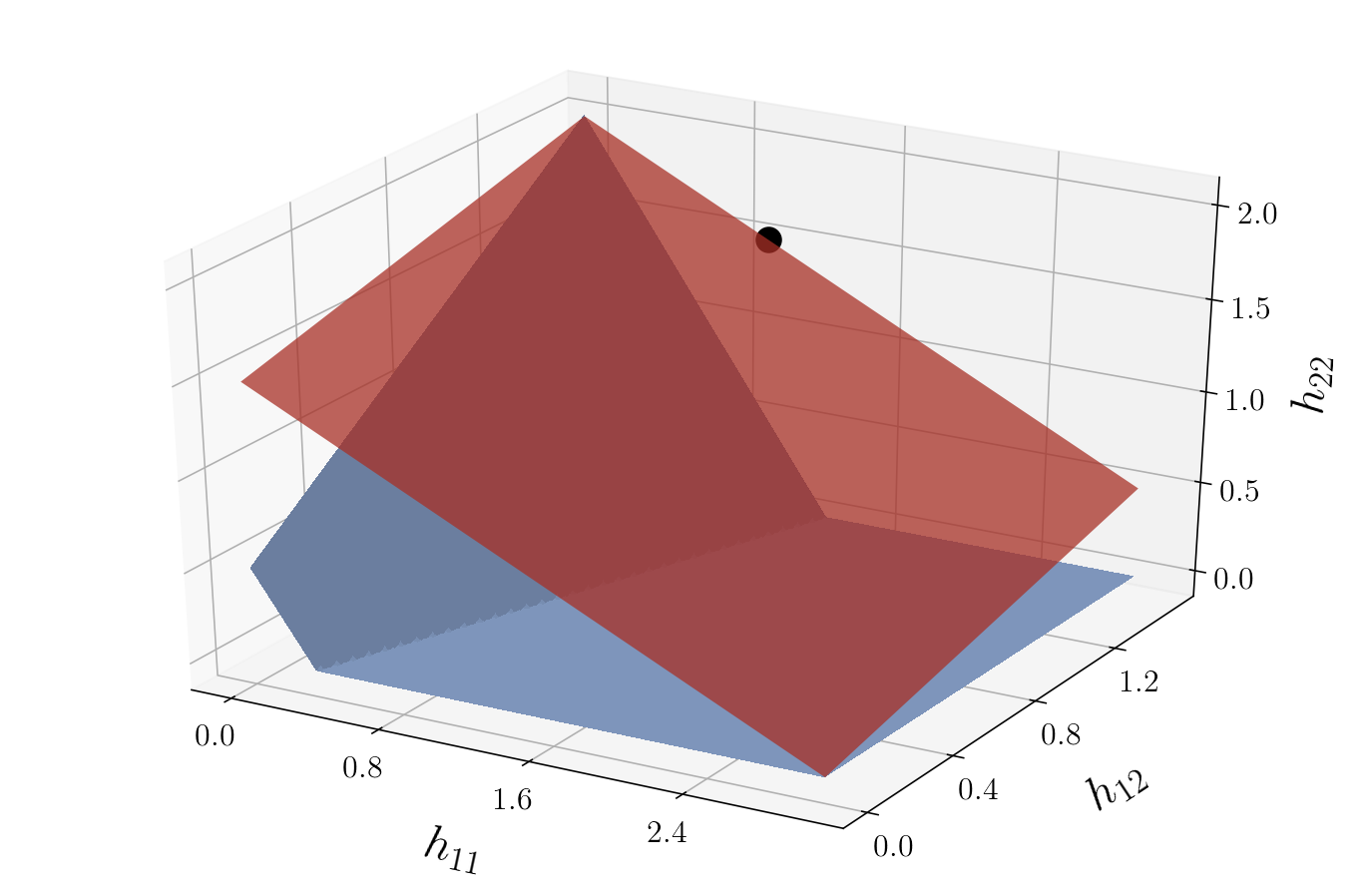}\caption{An inequality from the $\Delta$-relaxation.}
\end{subfigure}
\begin{subfigure}[t]{.3\linewidth}
\centering
\includegraphics[width=\textwidth]{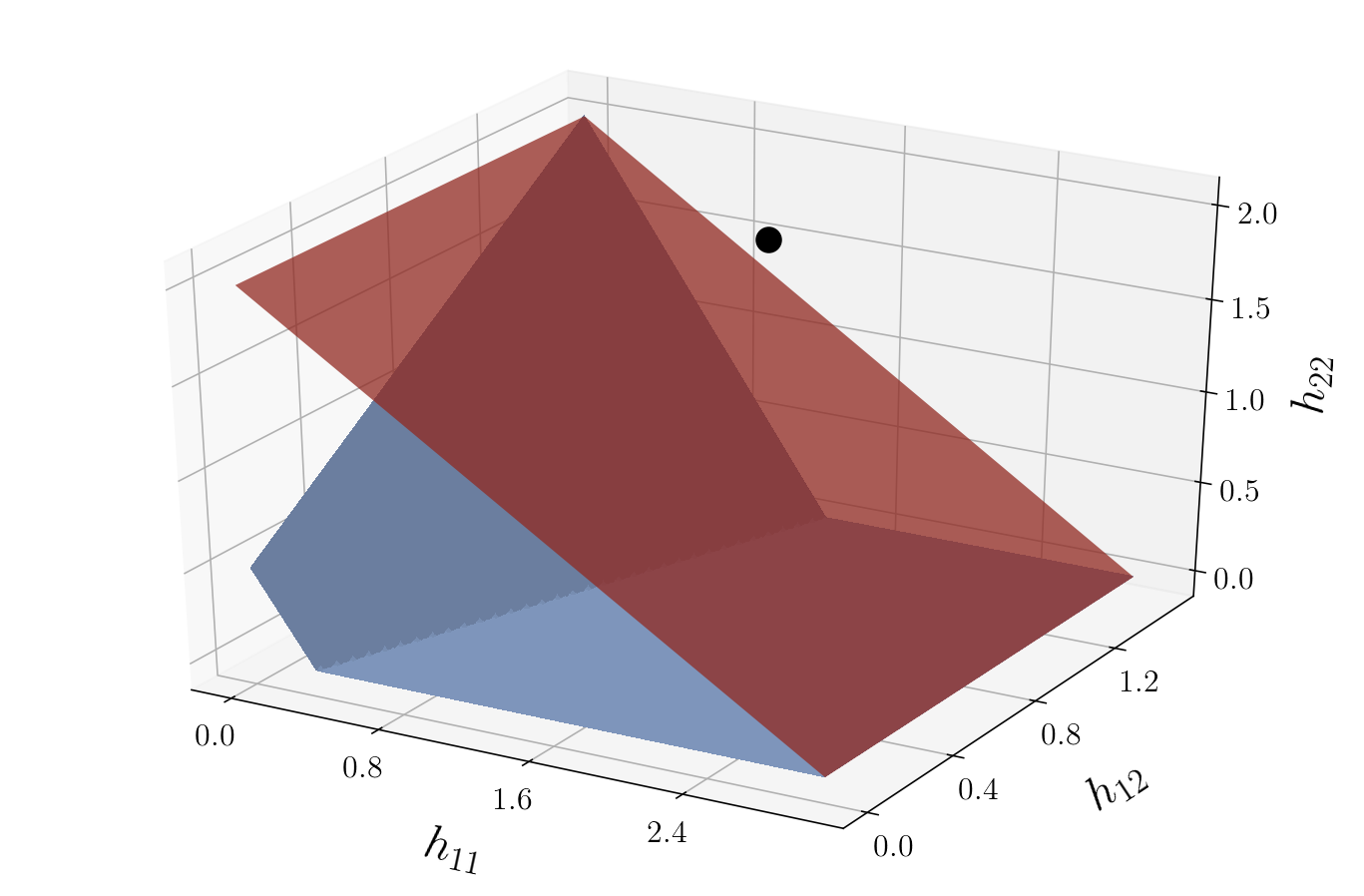}\caption{An inequality from~\eqref{eqn:cut-family-2}.}
\end{subfigure}
\begin{subfigure}[t]{.3\linewidth}
\centering
\includegraphics[width=\textwidth]{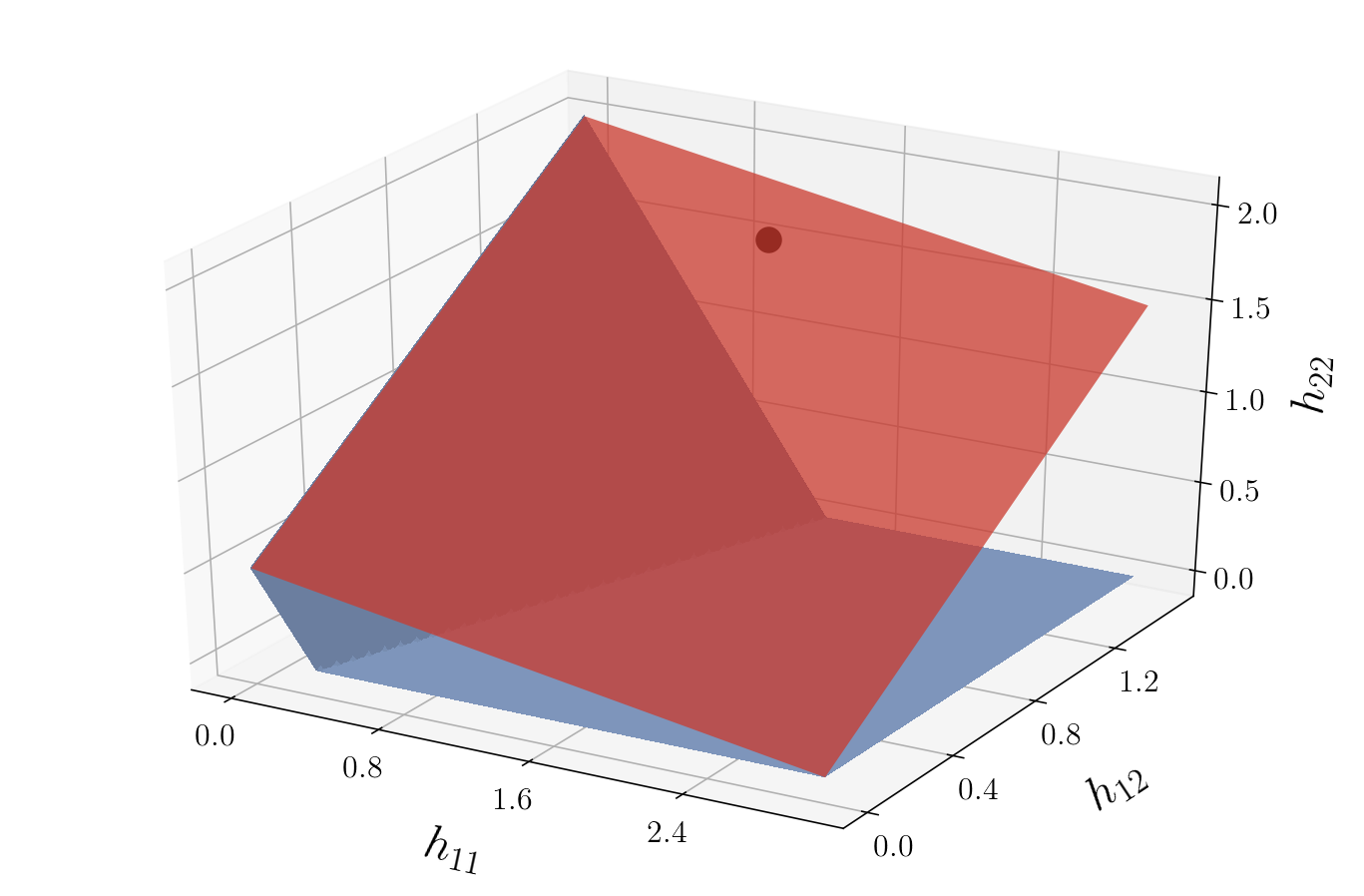}\caption{Another inequality from~\eqref{eqn:cut-family-2}.}
\end{subfigure}
\caption{Possible choices of upper bounding functions $\mathcal{U}_i$ for a single ReLU. The black point depicts a solution $z_{1:N}$ that we would like to separate (projected to the input-output space of the ReLU), which is cut off by the inequality in (b). A full example involving these particular inequalities can be found in Appendix~\ref{app:fastc2v_example}.}\label{fig:fastc2v_example_inequalities_main}
\end{figure}

In Theorem~\ref{thm:convex-hull} we have derived a family of inequalities, \eqref{eqn:cut-family-2}, which can be applied to yield valid upper bounding affine functions for each intermediate neuron in a network. As there may be exponentially many such inequalities, it is not clear \emph{a priori} which to select as input to the algorithm from Section~\ref{sec:generic_framework}. Therefore, we present a simple iterative scheme in which we apply a small number of solves of \eqref{eqn:relaxed-problem}, incrementally updating the set of affine bounding functions used at each iteration.

Our goal is to update the upper bounding function $\calU_i$ with one of the inequalities from \eqref{eqn:cut-family-2} as illustrated in Figure~\ref{fig:fastc2v_example_inequalities_main} via the separation procedure of \autoref{lemma:general_separation}, which requires an optimal solution $z_{1:N}$ for \eqref{eqn:relaxed-problem}. However, the backward pass of the propagation algorithm described in Section~\ref{sec:generic_framework} only computes the optimal value $B\rbra{\calC,\eta}$ and a partial solution $z_{1:m}$. For this reason, we first extend the propagation algorithm with a forward pass that completes the partial solution $z_{1:m}$ by propagating the values for $z_{m+1},\ldots,z_N$ through the network. This propagation uses the same affine bounding functions from \eqref{eqn:relaxedConstr} that were used to eliminate variables in the backward pass. For more details, see Algorithm~\ref*{alg:forwards} in \appendixref{propagationalgappendix}.


In essence, our complete dynamic algorithm initializes with a set of bounding functions (e.g.\ from \FastLin{} or \DeepPoly{}), applies a backward pass to solve the bounding problem, and then a forward pass to reconstruct the full solution. It then takes that full solution, and at each intermediate neuron $i$ applies the separation procedure of \autoref{lemma:general_separation} to produce an inequality from the family \eqref{eqn:cut-family-2}. If this inequality is violated, it replaces the upper bounding function $\calU_i$ with this inequality from \eqref{eqn:cut-family-2}. We repeat for as many iterations as desired and take the best bound produced across all iterations. In this way, we use separation to help us select from a large family just one inequality that will (hopefully) be most beneficial for improving the bound. For more details, see Algorithm~\ref*{alg:iterative} in \appendixref{propagationalgappendix}.



\section{Computational experiments}\label{sec:computational}

\subsection{Computational setup}

We evaluate two methods: the propagation-based algorithm from Section~\ref{sec:new_prop_algorithm} and a method based on partially solving the LP from Theorem~\ref{thm:convex-hull} by treating the inequalities~\eqref{eqn:cut-family-2} as cutting planes, i.e.\ inequalities that are dynamically added to tighten a relaxation. To focus on the benefit of incorporating the inequalities \eqref{eqn:cut-family-2} into verification algorithms, we implement simple versions of the algorithms, devoid of extraneous features and fine-tuning. We name this framework ``Cut-to-Verify'' (C2V), and the propagation-based and LP-based algorithms \FastCtV{} and \OptCtV{}, respectively. See \url{https://github.com/google-research/tf-opt} for the implementation.

The overall framework in both methods is the same: we compute scalar bounds for the pre-activation variables of all neurons as we move forward in the network, using those bounds to produce the subsequent affine bounding functions and LP formulations as discussed in Section~\ref{ssec:selecting-bounding-functions}. Below, we describe the bounds computation for each individual neuron.

\paragraph{Propagation-based algorithm\ (\FastCtV{}).} We implement the algorithm described in Section~\ref{sec:new_prop_algorithm}, using the initial affine bounding functions $\{\calL_i, \calU_i\}_{i=m+1}^N$ from \DeepPoly{}~\cite{singh2019abstract} and \texttt{CROWN-Ada}~\cite{zhang2018efficient}, as described in Section~\ref{sec:generic_framework}.\footnote{Our framework supports initializing from the \FastLin{} inequalities as well, but it has been observed that the inequalities from \DeepPoly{} perform better computationally.} In this implementation, we run a single iteration of the algorithm.


\paragraph{LP-based algorithm\ (\OptCtV{}).}
Each bound is generated by solving a series of LPs where our upper bounding inequalities are dynamically generated and added as cutting planes. We start with the standard $\Delta$-relaxation LP, solve it to optimality, and then for every neuron preceding the one we are bounding, we add the most violated inequality with respect to the LP optimum by solving~\eqref{separationproblem}. This can be repeated multiple times. In this implementation, we perform three rounds of separation. We generate new cuts from scratch for each bound that we compute.

\medskip

In both methods, at each neuron we take the best between the bound produced by the method and the trivial interval arithmetic bound. \appendixref{app:implementation_details} contains other implementation details.

We compare each of our novel algorithms against their natural baselines: \DeepPoly{} for our propagation-based method, and the standard $\Delta$-relaxation LP for our cutting plane method. Our implementation of \DeepPoly{} is slightly different from the one in~\cite{singh2019abstract} in that we take the best of interval arithmetic and the result of \DeepPoly{} at each neuron. Moreover, our implementation is sequential, even though operations in the same layer could be parallelized (for each of the algorithms implemented in this work). The LP method simply solves the $\Delta$-relaxation LP to generate bounds at each neuron. In addition, we compare them with \RefineZono{}~\cite{singh2019boosting} and \kPoly{}~\cite{singh2019beyond}, two state-of-the-art incomplete verification methods.




\paragraph{Verification problem.} We consider the following verification problem: given a correctly labeled target image, certify that the neural network returns the same label for each input within $L_{\infty}$-distance at most $\epsilon$ of that target image. More precisely, given an image $\hat{x} \in [0,1]^m$ correctly labeled as $t$, a neural network where $f_k(x)$ returns its logit for class $k \in K$, and a distance $\epsilon > 0$, the image $\hat{x}$ is verified to be robust if $\max_{x \in [\hat{L}, \hat{U}]} \max_{k \in K} \{ f_k(x) - f_t(x)\} < 0$, where $\hat{L}_i = \max\{0, \hat{x}_i - \epsilon\}$ and $\hat{U}_i = \min\{1, \hat{x}_i + \epsilon\}$ for all $i = 1, \ldots, m$. For propagation-based methods, the inner $\max$ term can be handled by computing bounds for $f_k(x) - f_t(x)$ for every class $k \neq t$ and checking if the maximum bound is negative, although we only need to compute pre-activation bounds throughout the network once. For LP-based methods, this inner term can be incorporated directly into the model.


To facilitate the comparison with existing algorithms, our experimental setup closely follows that of Singh et al.~\cite{singh2019beyond}. We experiment on a subset of trained neural networks from the publicly available ERAN dataset~\cite{eran_benchmark}. We examine the following networks: the fully connected ReLU networks 6x100 ($\epsilon = 0.026$), 9x100 ($\epsilon = 0.026$), 6x200 ($\epsilon = 0.015$), 9x200 ($\epsilon = 0.015$), all trained on MNIST without adversarial training; the ReLU convolutional networks ConvSmall for MNIST ($\epsilon = 0.12$), with 3 layers and trained without adversarial training; the ReLU network ConvBig for MNIST ($\epsilon = 0.3$), with 6 layers and trained with DiffAI; and the ReLU network ConvSmall for CIFAR-10 ($\epsilon = 2/255$), with 3 layers and trained with PGD. These $\epsilon$ values are the ones used in~\cite{singh2019beyond} and they are cited as being challenging. For more details on these networks, see Appendix~\ref{app:implementation_details} or~\cite{eran_benchmark}. For each network, we verify the first 1000 images from their respective test sets except those that are incorrectly classified.

Due to numerical issues with LPs, we zero out small values in the convolutional networks for the LP-based algorithms (see Appendix~\ref{app:implementation_details}). Other than this, we do not perform any tuning according to instance. Our implementation is in C++ and we perform our experiments in an Intel Xeon E5-2699 2.3Ghz machine with 128GB of RAM. We use Gurobi 8.1 as the LP solver, take advantage of incremental solves, and set the LP algorithm to dual simplex, as we find it to be faster for these LPs in practice. This means that our LP implementation does not run in polynomial time, even though it could in theory by using a different LP algorithm (see Corollary~\ref{polylemma:ref}).

To contextualize the results, we include an upper bound on the number of verifiable images. This is computed with a standard implementation of gradient descent with learning rate 0.01 and 20 steps. For each image, we take 100 random initializations (10 for MNIST ConvBig and CIFAR-10 ConvSmall) and check if the adversarial example produced by gradient descent is valid. The upper bound is the number of images for which we were unable to produce an adversarial example.

\subsection{Computational results}

\begin{table}[tpb]

    \caption{Number of images verified and average verification times per image for a set of networks from the ERAN dataset~\cite{eran_benchmark}. ConvS and ConvB denote ConvSmall and ConvBig, respectively. Results for \RefineZono{} and \kPoly{} are taken from~\cite{singh2019beyond}.}
    \label{tab:computational_results}
    \centering

    \begin{tabular}{llrrrrrrr}
        \toprule
        & & \multicolumn{6}{c}{MNIST} & \multicolumn{1}{c}{CIFAR-10} \\\cmidrule(lr){3-8}\cmidrule(lr){9-9}
        Method    & & 6x100 & 9x100 & 6x200 & 9x200 & ConvS & ConvB & ConvS\\\midrule
        \multirow{2}{*}{\DeepPoly}  & \#verified &   160 &   182 &   292 &  259 &  162 &   652 &   359 \\
                  & Time (s)   &   0.7 &   1.4 &   2.4 &  5.6 &  0.9 &   7.4 &   2.8 \\[0.25em]
        \multirow{2}{*}{\FastCtV} & \#verified &   279 &   269 &   477 &  392 &  274 &   691 &   390 \\
                  & Time (s)   &   8.7 &  19.3 &  25.2 & 57.2 &  5.3 &  16.3 &  15.3 \\[0.25em]
        \multirow{2}{*}{\texttt{LP}}        & \#verified &   201 &   223 &   344 &  307 &  242 &   743 &   373 \\
                  & Time (s)   &  50.5 & 385.6 & 218.2 & 2824.7 & 23.1 &  24.9 &  38.1 \\[0.25em]
        \multirow{2}{*}{\OptCtV} & \#verified &   429 &   384 &   601 &  528 &  436 &   771 &   398 \\
                  & Time (s)   & 136.7 & 759.4 & 402.8 & 3450.7 & 55.4 & 102.0 & 104.8 \\\midrule
        \RefineZono & \#verified & 312 & 304 & 341 & 316 & 179 & 648 & 347 \\
        \kPoly        & \#verified & 441 & 369 & 574 & 506 & 347 & 736 & 399 \\\midrule
        Upper bound &\#verified & 842 & 820 & 901 & 911 & 746 & 831 & 482\\
                 \bottomrule
    \end{tabular}\vspace{1em}

\end{table}

The computational results in Table~\ref{tab:computational_results} demonstrate that adding the upper bounding inequalities proposed in this paper significantly improves the number of images verified compared to their base counterparts. While on average \FastCtV{} spends an order of magnitude more time than \DeepPoly{} to achieve this, it still takes below one minute on average for all instances examined. \OptCtV{} takes approximately 1.2 to 2.7 times of a pure LP method to generate bounds in the problems examined. Since we start from the LP basis of the previous solve, subsequent LPs after adding cuts are generally faster.

Interestingly, we observe that \FastCtV{} verifies more images than LP in almost all cases in much less time. This indicates that, in practice, a two-inequality relaxation with a single (carefully chosen) tighter inequality from \eqref{eqn:cut-family-2} can often be stronger than the three-inequality $\Delta$-relaxation.

When compared to other state-of-the-art incomplete verifiers, we observe that for the larger networks, improving \DeepPoly{} with our inequalities enables it to verify more images than \RefineZono{}~\cite{singh2019boosting}, a highly fine-tuned method that combines MIP, LP, and \DeepPoly{}, but without the expensive computation and the parameter tuning needs from \RefineZono{}. In addition, we find that adding our inequalities to LPs is competitive with \kPoly{}, surpassing it for some of the networks.
While the timings in~\cite{singh2019beyond} may not be comparable to our timings, the authors report average times for \RefineZono{} and \kPoly{} within the range of 4 to 15 minutes and 40 seconds to 8 minutes, respectively.


Appendix~\ref{app:supplementary} contains additional computational results where we consider multiple trained networks and distances $\epsilon$ from the base image.

\paragraph{Outlook: Our methods as subroutines}The scope of our computational experiments is to demonstrate the practicality and strength of our full-neuron relaxation applied to simple methods, rather than to engineer full-blown state-of-the-art verification methods. Towards such a goal, we remark that both \RefineZono{} and \kPoly{} rely on LP and other faster verification methods as building blocks to a stronger method, and either of our methods could be plugged into them. For example, we could consider a hybrid approach similar to \RefineZono{} that uses the stronger, but slower  \OptCtV{} in the earlier layers (where it can have the most impact) and then switches to  \FastCtV{}, which could result in verification times closer to \FastCtV{} with an effectiveness closer to \OptCtV{}. In addition, \kPoly{} exploits the correlation between multiple neurons in the same layer, whereas our approach does not, suggesting that there is room to combine approaches. Finally, we note that solving time can be controlled with a more careful management of the inequalities to be added and parallelizing bound computation of neurons in the same layer.


\newpage
\section*{Broader Impact}




In a world where deep learning is impacting our lives in ever more tangible ways, verification is an essential task to ensure that these black box systems behave as we expect them to. Our fast, simple algorithms have the potential to make a positive impact by verifying a larger number of inputs to be robust within a short time-frame, often required in several applications.
Of course, we should be cautious that although our algorithms provide a mathematical certificate of an instance being robust, failure to use the system correctly, such as modeling the verification problem in a way that does not reflect real-world concerns, can still lead to unreliable neural networks. We also highlight that our version of the verification problem, while accurately capturing a reasonable formal specification of robustness, clearly does not perfectly coincide with ``robustness'' as may be used in a colloquial sense. Therefore, we highlight the importance of understanding the strengths and limitations of the mathematical model of verification used, so that a false sense of complacency does not set in.



\bibliographystyle{plain}
\bibliography{master_editable.bib}

\begin{thebibliography}{10}

\bibitem{ahmed2011maximizing}
Shabbir Ahmed and Alper Atamt{\"u}rk.
\newblock Maximizing a class of submodular utility functions.
\newblock {\em Mathematical programming}, 128(1-2):149--169, 2011.

\bibitem{anderson2020tightened}
Brendon~G Anderson, Ziye Ma, Jingqi Li, and Somayeh Sojoudi.
\newblock Tightened convex relaxations for neural network robustness
  certification.
\newblock {\em arXiv preprint arXiv:2004.00570}, 2020.

\bibitem{anderson2020strong}
Ross Anderson, Joey Huchette, Will Ma, Christian Tjandraatmadja, and Juan~Pablo
  Vielma.
\newblock Strong mixed-integer programming formulations for trained neural
  networks.
\newblock {\em Mathematical Programming}, pages 1--37, 2020.

\bibitem{Strong-mixed-integer-programming-formulations-for-trained-CONF}
Ross Anderson, Joey Huchette, Christian Tjandraatmadja, and Juan~Pablo Vielma.
\newblock {Strong mixed-integer programming formulations for trained neural
  networks}.
\newblock In A.~Lodi and V.~Nagarajan, editors, {\em Proceedings of the 20th
  Conference on Integer Programming and Combinatorial Optimization ({IPCO
  2019})}, volume 11480 of {\em Lecture Notes in Computer Science}, pages
  27--42, 2019.

\bibitem{anthony2001discrete}
Martin Anthony.
\newblock {\em Discrete mathematics of neural networks: selected topics},
  volume~8.
\newblock SIAM, 2001.

\bibitem{bach2013learning}
Francis Bach.
\newblock Learning with submodular functions: A convex optimization
  perspective.
\newblock {\em Foundations and Trends{\textregistered} in Machine Learning},
  6(2-3):145--373, 2013.

\bibitem{Bertsimas:1997}
Dimitris Bertsimas and John Tsitsiklis.
\newblock {\em Introduction to Linear Optimization}.
\newblock Athena Scientific, 1997.

\bibitem{bhojanapalli2020efficient}
Srinadh Bhojanapalli, Rudy Bunel, Krishnamurthy Dvijotham, and Oliver Hinder.
\newblock An efficient nonconvex reformulation of stagewise convex optimization
  problems.
\newblock In {\em Advances in Neural Information Processing Systems}, 2020.

\bibitem{botoevaefficient}
Elena Botoeva, Panagiotis Kouvaros, Jan Kronqvist, Alessio Lomuscio, and Ruth
  Misener.
\newblock Efficient verification of {ReLU}-based neural networks via dependency
  analysis.
\newblock In {\em Thirty-Fourth AAAI Conference on Artificial Intelligence},
  2020.

\bibitem{bunel2020branch}
Rudy Bunel, Jingyue Lu, Ilker Turkaslan, P~Kohli, P~Torr, and P~Mudigonda.
\newblock Branch and bound for piecewise linear neural network verification.
\newblock {\em Journal of Machine Learning Research}, 21(2020), 2020.

\bibitem{Carlini:2016}
Nicholas {Carlini} and David {Wagner}.
\newblock Towards evaluating the robustness of neural networks.
\newblock In {\em 2017 IEEE Symposium on Security and Privacy (SP)}, pages
  39--57, 2017.

\bibitem{cheng2017maximum}
Chih-Hong Cheng, Georg N{\"u}hrenberg, and Harald Ruess.
\newblock Maximum resilience of artificial neural networks.
\newblock In {\em International Symposium on Automated Technology for
  Verification and Analysis}, pages 251--268. Springer, 2017.

\bibitem{conforti2014integer}
Michele Conforti, G{\'e}rard Cornu{\'e}jols, and Giacomo Zambelli.
\newblock {\em Integer programming}, volume 271.
\newblock Springer, 2014.

\bibitem{dutta2018output}
Souradeep Dutta, Susmit Jha, Sriram Sankaranarayanan, and Ashish Tiwari.
\newblock Output range analysis for deep feedforward neural networks.
\newblock In {\em NASA Formal Methods Symposium}, pages 121--138. Springer,
  2018.

\bibitem{dvijotham2018dual}
Krishnamurthy Dvijotham, Robert Stanforth, Sven Gowal, Timothy~A Mann, and
  Pushmeet Kohli.
\newblock A dual approach to scalable verification of deep networks.
\newblock In {\em UAI}, volume~1, page~2, 2018.

\bibitem{ehlers2017formal}
Ruediger Ehlers.
\newblock Formal verification of piece-wise linear feed-forward neural
  networks.
\newblock In {\em International Symposium on Automated Technology for
  Verification and Analysis}, pages 269--286. Springer, 2017.

\bibitem{fischetti2017deep}
Matteo Fischetti and Jason Jo.
\newblock Deep neural networks as 0-1 mixed integer linear programs: A
  feasibility study.
\newblock {\em Constraints}, 23:296--309, 2018.

\bibitem{grotschel2012geometric}
Martin Gr{\"o}tschel, L{\'a}szl{\'o} Lov{\'a}sz, and Alexander Schrijver.
\newblock {\em Geometric algorithms and combinatorial optimization}, volume~2.
\newblock Springer Science \& Business Media, 2012.

\bibitem{huang2017safety}
Xiaowei Huang, Marta Kwiatkowska, Sen Wang, and Min Wu.
\newblock Safety verification of deep neural networks.
\newblock In {\em International Conference on Computer Aided Verification},
  pages 3--29. Springer, 2017.

\bibitem{Katz:2017}
Guy Katz, Clark Barrett, David~L. Dill, Kyle Julian, and Mykel~J. Kochenderfer.
\newblock Reluplex: {A}n efficient {SMT} solver for verifying deep neural
  networks.
\newblock In {\em International Conference on Computer Aided Verification},
  pages 97--117, 2017.

\bibitem{Korte:2000}
Bernhard Korte and Jens Vygen.
\newblock {\em Combinatorial Optimization: Theory and Algorithms}.
\newblock Springer, 2000.

\bibitem{liu2019algorithms}
Changliu Liu, Tomer Arnon, Christopher Lazarus, Clark Barrett, and Mykel~J
  Kochenderfer.
\newblock Algorithms for verifying deep neural networks.
\newblock {\em arXiv preprint arXiv:1903.06758}, 2019.

\bibitem{liu2019training}
Chen Liu, Mathieu Salzmann, and Sabine S{\"u}sstrunk.
\newblock Training provably robust models by polyhedral envelope
  regularization.
\newblock {\em arXiv}, pages arXiv--1912, 2019.

\bibitem{lomuscio2017approach}
Alessio Lomuscio and Lalit Maganti.
\newblock An approach to reachability analysis for feed-forward {ReLU} neural
  networks.
\newblock {\em arXiv preprint arXiv:1706.07351}, 2017.

\bibitem{lu2020neural}
Jingyue Lu and M.~Pawan~Kumar.
\newblock Neural network branching for neural network verification.
\newblock In {\em International Conference on Learning Representations}, 2020.

\bibitem{lyu2019fastened}
Zhaoyang Lyu, Ching-Yun Ko, Zhifeng Kong, Ngai Wong, Dahua Lin, and Luca
  Daniel.
\newblock Fastened {CROWN}: Tightened neural network robustness certificates.
\newblock {\em arXiv preprint arXiv:1912.00574}, 2019.

\bibitem{narodytska2018verifying}
Nina Narodytska, Shiva Kasiviswanathan, Leonid Ryzhyk, Mooly Sagiv, and Toby
  Walsh.
\newblock Verifying properties of binarized deep neural networks.
\newblock In {\em Thirty-Second AAAI Conference on Artificial Intelligence},
  2018.

\bibitem{o1971hyperplane}
Patrick~E O'Neil.
\newblock Hyperplane cuts of an n-cube.
\newblock {\em Discrete Mathematics}, 1(2):193--195, 1971.

\bibitem{Papernot:2016}
Nicolas Papernot, Patrick McDaniel, Somesh Jha, Matt Fredrikson, Z.~Berkay
  Celik, and Ananthram Swami.
\newblock The limitations of deep learning in adversarial settings.
\newblock In {\em IEEE European Symposium on Security and Privacy}, pages
  372--387, March 2016.

\bibitem{raghunathan2018semidefinite}
Aditi Raghunathan, Jacob Steinhardt, and Percy~S Liang.
\newblock Semidefinite relaxations for certifying robustness to adversarial
  examples.
\newblock In {\em Advances in Neural Information Processing Systems}, pages
  10877--10887, 2018.

\bibitem{Roessig2019}
Ansgar R{\"o}ssig.
\newblock Verification of neural networks.
\newblock Technical Report 19-40, ZIB, Takustr. 7, 14195 Berlin, 2019.

\bibitem{salman2019convex}
Hadi Salman, Greg Yang, Huan Zhang, Cho-Jui Hsieh, and Pengchuan Zhang.
\newblock A convex relaxation barrier to tight robustness verification of
  neural networks.
\newblock In {\em Advances in Neural Information Processing Systems}, pages
  9832--9842, 2019.

\bibitem{scheibler2015towards}
Karsten Scheibler, Leonore Winterer, Ralf Wimmer, and Bernd Becker.
\newblock Towards verification of artificial neural networks.
\newblock In {\em MBMV}, pages 30--40, 2015.

\bibitem{singh2019beyond}
Gagandeep Singh, Rupanshu Ganvir, Markus P{\"u}schel, and Martin Vechev.
\newblock Beyond the single neuron convex barrier for neural network
  certification.
\newblock In {\em Advances in Neural Information Processing Systems}, pages
  15072--15083, 2019.

\bibitem{singh2018fast}
Gagandeep Singh, Timon Gehr, Matthew Mirman, Markus P{\"u}schel, and Martin
  Vechev.
\newblock Fast and effective robustness certification.
\newblock In {\em Advances in Neural Information Processing Systems}, pages
  10802--10813, 2018.

\bibitem{singh2019abstract}
Gagandeep Singh, Timon Gehr, Markus P{\"u}schel, and Martin Vechev.
\newblock An abstract domain for certifying neural networks.
\newblock {\em Proceedings of the ACM on Programming Languages}, 3(POPL):1--30,
  2019.

\bibitem{singh2019boosting}
Gagandeep Singh, Timon Gehr, Markus P{\"u}schel, and Martin Vechev.
\newblock Boosting robustness certification of neural networks.
\newblock In {\em International Conference on Learning Representations}, 2019.

\bibitem{eran_benchmark}
Gagandeep Singh, Jonathan Maurer, Christoph Müller, Matthew Mirman, Timon
  Gehr, Adrian Hoffmann, Petar Tsankov, Dana Drachsler~Cohen, Markus Püschel,
  and Martin Vechev.
\newblock {ERAN} verification dataset.
\newblock \url{https://github.com/eth-sri/eran}.

\bibitem{Szegedy:2013}
Christian Szegedy, Wojciech Zaremba, Ilya Sutskever, Joan Bruna, Dumitru Erhan,
  Ian Goodfellow, and Rob Fergus.
\newblock Intriguing properties of neural networks.
\newblock In {\em International Conference on Learning Representations}, 2014.

\bibitem{tawarmalani2013explicit}
Mohit Tawarmalani, Jean-Philippe~P Richard, and Chuanhui Xiong.
\newblock Explicit convex and concave envelopes through polyhedral
  subdivisions.
\newblock {\em Mathematical Programming}, 138(1-2):531--577, 2013.

\bibitem{Tjeng:2017}
Vincent Tjeng, Kai Xiao, and Russ Tedrake.
\newblock Verifying neural networks with mixed integer programming.
\newblock In {\em International Conference on Learning Representations}, 2019.

\bibitem{todd1976computation}
M.J. Todd.
\newblock {\em The Computation of Fixed Points and Applications}.
\newblock Lecture Notes in Mathematics; 513. Springer-Verlag, 1976.

\bibitem{wang2018efficient}
Shiqi Wang, Kexin Pei, Justin Whitehouse, Junfeng Yang, and Suman Jana.
\newblock Efficient formal safety analysis of neural networks.
\newblock In {\em Advances in Neural Information Processing Systems}, pages
  6367--6377, 2018.

\bibitem{weng2018towards}
Lily Weng, Huan Zhang, Hongge Chen, Zhao Song, Cho-Jui Hsieh, Luca Daniel,
  Duane Boning, and Inderjit Dhillon.
\newblock Towards fast computation of certified robustness for {R}e{LU}
  networks.
\newblock In Jennifer Dy and Andreas Krause, editors, {\em Proceedings of the
  35th International Conference on Machine Learning}, volume~80 of {\em
  Proceedings of Machine Learning Research}, pages 5276--5285,
  Stockholmsmässan, Stockholm Sweden, 10--15 Jul 2018. PMLR.

\bibitem{wong2018provable}
Eric Wong and J.~Zico Kolter.
\newblock Provable defenses against adversarial examples via the convex outer
  adversarial polytope.
\newblock In {\em International Conference on Machine Learning}, pages
  5286--5295, 2018.

\bibitem{Wong:2018}
Eric Wong, Frank Schmidt, Jan~Hendrik Metzen, and J.~Zico Kolter.
\newblock Scaling provable adversarial defenses.
\newblock In {\em 32nd Conference on Neural Information Processing Systems},
  2018.

\bibitem{xiang2018output}
Weiming Xiang, Hoang-Dung Tran, and Taylor~T Johnson.
\newblock Output reachable set estimation and verification for multilayer
  neural networks.
\newblock {\em IEEE Transactions on Neural Networks and Learning Systems},
  29(11):5777--5783, 2018.

\bibitem{Xiao:2018}
Kai~Y. Xiao, Vincent Tjeng, Nur~Muhammad Shafiullah, and Aleksander Madry.
\newblock Training for faster adversarial robustness verification via inducing
  {ReLU} stability.
\newblock In {\em International Conference on Learning Representations}, 2019.

\bibitem{zhang2018efficient}
Huan Zhang, Tsui-Wei Weng, Pin-Yu Chen, Cho-Jui Hsieh, and Luca Daniel.
\newblock Efficient neural network robustness certification with general
  activation functions.
\newblock In {\em Advances in neural information processing systems}, pages
  4939--4948, 2018.

\bibitem{zhu2020improving}
Chen Zhu, Renkun Ni, Ping-yeh Chiang, Hengduo Li, Furong Huang, and Tom
  Goldstein.
\newblock Improving the tightness of convex relaxation bounds for training
  certifiably robust classifiers.
\newblock {\em arXiv preprint arXiv:2002.09766}, 2020.

\end{thebibliography}
\vfill
\pagebreak
\appendix
\begin{section}{Proof of Theorem~\ref{thm:convex-hull}}\label{theooneproof}

We provide two different proofs for Theorem~\ref{thm:convex-hull}. The first proof is based on classical machinery from submodular and convex optimization. The alternative proof is based on projecting down an extended MIP formulation built using disjunctive programming. We include them both since each proof provides unique insights on our new relaxation.

We first state a lemma that is used by both proofs for bounding the number of inequalities. Notationally, we will take $\bfzero^d$ and $\bfone^d$ as the length $d$ vectors of all zeros and all ones, respectively, and $e(i)\in \bbR^n$ for $i\in \sidx{n}$ as the $i$-th canonical unit vector, where the length will be implicitly determined by the context of its use. In some cases it will be convenient to refer to the $0$-th canonical vector $e(0)=\bfzero^n$.

\begin{lemma}\label{hyperplanelemma}
If $u,v\in \{0,1\}^\effd$ are such that $\sum_{i=1}^\effd \abs{u_i-v_i}=1$, then we say that $uv$ is an edge of $[0,1]^\effd$. For $w\in \bbR^\effd$ and $b\in \bbR$, we say the hyperplane $w\cdot x+b=0$ cuts edge $uv$  of $[0,1]^\effd$ if  $w\cdot u+b <0$ and $w\cdot v+b \geq 0$. If $b<0$ and $\sum_{i=1}^\effd w_i +b \geq 0$, then the number of edges cut by one such hyperplane is lower-bounded by $\effd$ and upper-bounded by $\lceil \frac{1}{2} \effd \rceil \binom{\effd}{\lceil \frac{1}{2} \effd \rceil}$.
For each bound there exists a hyperplane with $w\in \bbR^\effd_+$ such that the bounds holds at equality.
\end{lemma}
\begin{proof}
Consider the graph $G=(V,E)$ with $V=\{0,1\}^\effd$ and $E$ equal to the edges of $[0,1]^\effd$. Let $s = \bfzero^d$ and $t = \bfone^d$. Then  $w\cdot s+b<0$ and $w\cdot t +b \geq 0$, so the edges of $[0,1]^\effd$ cut by the hyperplane form a $s-t$ graph-cut in $G$ (note that this does \emph{not} have the same meaning as the definition of cut for an edge given in the Lemma statement). Hence, the number of edges cut by the hyperplane are lower bounded by $\effd$ (e.g.\ follows by Menger's theorem
by noting that there are $\effd$ disjoint paths in $G$ from $s$ to $t$). An example of a hyperplane that achieves this lower bound is $w=\bfone^d$ and $b=-1/2$.

The tight upper bound follows from a simple adaptation of the proof of a result from \cite{o1971hyperplane}.\footnote{See also \cite[Theorem 7.9]{anthony2001discrete}: the proof of \cite[Lemma 2]{o1971hyperplane} can be readily adapted to accommodate non-strict, rather than strict, linear inequalities.} An example of a hyperplane that achieves this upper bound is $w=\bfone^d$ and $b=-\lceil \frac{1}{2} \effd \rceil$.
\end{proof}

\begin{subsection}{A proof using submodularity}
We start with an example.
\subsubsection{Illustrative example and definitions}
\begin{example}\label{simpleexample}
Consider the set from \eqref{Sdef:eq} for $n=2$, $w = (1,1)$, $b = (-1.5)$, $L=(0,0)$ and $U=(0,0)$, which corresponds to
\[ S=\Set{\rbra{x,y}\in [0,1]^2\times \bbR| y=g(x)}\]
for $g(x)\defeq\max\Set{0,x_1+x_2-1.5}$. Set $S$ is depicted in \autoref{fig:relunew} and we can check that $\Conv(S)$ is described by
\begin{subequations}\label{simpleexample:eq}
\begin{align}
    x&\in [0,1]^2\\
    y&\geq g(x)\label{relunew:eq1}\\
    y\leq r_1(x),\quad  y&\leq r_2(x)\label{relunew:eq2}
\end{align}
\end{subequations}
for $r_1(x)\defeq 0.5x_2$ and $r_2(x)\defeq 0.5x_1$. Inequality \eqref{relunew:eq1} is obtained by relaxing the equation describing $S$ to an inequality and using the fact that $g(x)$ is convex. Functions $r_1$ and $r_2$ from inequality \eqref{relunew:eq2} are depicted in Figures~\ref{fig:relunew:a} and \ref{fig:relunew:b}, respectively. These functions can be obtained through the following interpolation procedure.

First, consider the subdivision of $[0,1]^2$ into the triangles $\mathbf{T}_{1}$ and $\mathbf{T}_{2}$ depicted in Figures~\ref{fig:relunew:a} and \ref{fig:relunew:b}, respectively. As depicted  \autoref{fig:relunew:a}, the vertices of   $\mathbf{T}_{1}$ are obtained by incrementally adding the canonical vectors to $(0,0)$, in order, until we obtain  $(1,1)$. That is, the vertices of  $\mathbf{T}_{1}$ are $e(0)=(0,0)$, $e(0)+e(1)=e(1)=(0,1)$ and $e(0)+e(1)+e(2)=(1,1)$.  In contrast, as depicted in \autoref{fig:relunew:b}, the vertices of   $\mathbf{T}_{2}$ are obtained by incrementally adding the canonical vectors in reverse order (i.e. the vertices of  $\mathbf{T}_{2}$ are $e(0)=(0,0)$, $e(0)+e(2)=(1,0)$ and $e(0)+e(2)+e(1)=(1,1)$).

Second,  we obtain $r_1$ and $r_2$ by constructing the unique affine interpolation of $g$ on $\mathbf{T}_{1}$ and $\mathbf{T}_{2}$, respectively. That is, as depicted in  \autoref{fig:relunew:a},  $r_1(x)=\alpha^1\cdot x+ \beta_1 $, where $\alpha^1\in \bbR^2$ and $\beta_1\in \bbR$ are such that $r_1$ is equal to $g$ for the three vertices $(0,0)$, $(1,0)$ and $(1,1)$ of $\mathbf{T}_{1}$:
\[\begin{pmatrix}0 & 0\\ 1 & 0\\ 1 & 1\end{pmatrix}\alpha^1 + \beta_1 =\begin{pmatrix} g(0,0)\\g(1,0)\\g(1,1) \end{pmatrix} = \begin{pmatrix} 0\\0\\0.5 \end{pmatrix}.\]
The unique solution of this system is $\alpha^1=(0,0.5)$ and $\beta_1=0$, which yields $r_1(x)=0.5x_2$. Function $r_2$ is obtained by a similar procedure using the vertices of $\mathbf{T}_{2}$ as illustrated  \autoref{fig:relunew:b}.
\end{example}

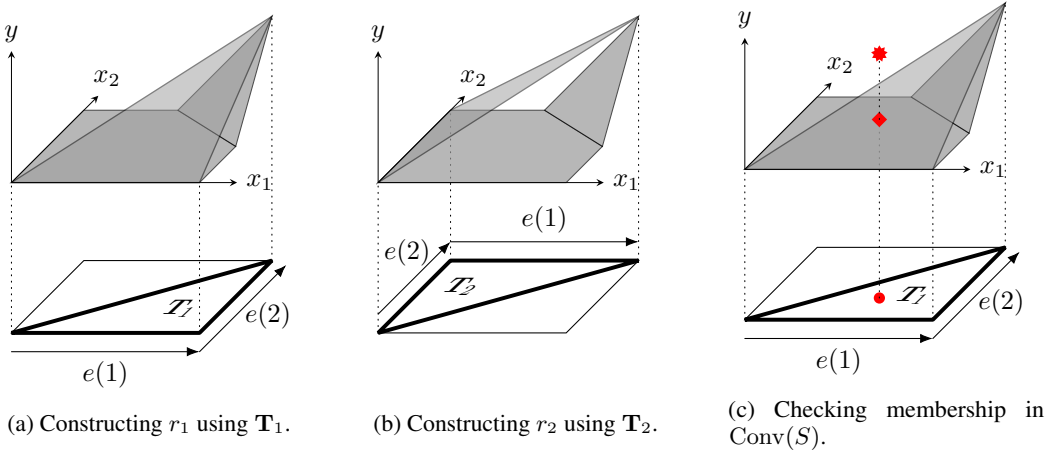
\begin{figure}[htpb]
    \centering
          \begin{subfigure}{.3\linewidth}
      \begin{tikzpicture}[yzx]
        \draw [->, >=stealth, line width=0.25] (0,0,0) -- (1.2,0,0);
        \draw [->, >=stealth, line width=0.25] (0,0,0) -- (0,1.2,0);
        \draw [->, >=stealth, line width=0.25] (0,0,0) -- (0,0,0.7);
        \node[above right] at (1.2,-.075,0) {$x_2$};
        \node[right] at (0,1.2,0) {$x_1$};
        \node[above] at (0,0,.7) {$y$};
        \coordinate (LL) at (0,0,0);
        \coordinate (UL) at (1,0,0);
        \coordinate (LU) at (0,1,0);
        \coordinate (UU) at (1,1,0.5);
        \coordinate (UM) at (1,0.5,0);
        \coordinate (MU) at (0.5,1,0);

        \draw [fill=gray!80,opacity=0.7] (LL) -- (UL) -- (UM) -- (MU) -- (LU) -- cycle;
        \draw [fill=gray!80,opacity=0.7] (UU) -- (UM) -- (MU) -- cycle;

        \def\gridlevel{-0.8}
        \def\arrowlevel{-0.9}
        \def\arrowuplevel{-0.7}
        \draw (0,0,\gridlevel) -- (0,1,\gridlevel) -- (1,1,\gridlevel) -- (1,0,\gridlevel) -- cycle;
        \draw [fill=gray!80,opacity=0.5,line width=0.5pt] (LL) -- (LU) -- (UU) -- cycle;
        \draw[line width=1.5pt,line join=bevel] (0,0,\gridlevel) -- (0,1,\gridlevel) -- (1,1,\gridlevel)  -- cycle;
        \draw[dotted,line cap=round] (0,0,\gridlevel) -- (0,0,0);
        \draw[dotted,line cap=round] (0,1,\gridlevel) -- (0,1,0);
        \draw[dotted,line cap=round] (1,1,\gridlevel) -- (1,1,0.5);
        \draw[line cap=round,-Latex] (0,0,\arrowlevel) -- (0,1,\arrowlevel);
        \draw[line cap=round,-Latex] (0,1,\arrowlevel) -- (1.2,1,\arrowlevel);
        \node[below] at (0,0.5,\arrowlevel) {$e(1)$};
        \node[right] at (0.5,1,\arrowlevel) {$e(2)$};
        \node[yscale=0.8,xslant=0.85] (A) at (0.6,0.67,\arrowlevel) {  $\mathbf{T}_{1}$};


    \end{tikzpicture}
      \captionsetup{font=small}
        \caption{Constructing $r_1$ using $\mathbf{T}_{1}$.  \label{fig:relunew:a}}
      \end{subfigure}
 \hspace{0.2in}
          \begin{subfigure}{.3\linewidth}
      \begin{tikzpicture}[yzx]
        \draw [->, >=stealth, line width=0.25] (0,0,0) -- (1.2,0,0);
        \draw [->, >=stealth, line width=0.25] (0,0,0) -- (0,1.2,0);
        \draw [->, >=stealth, line width=0.25] (0,0,0) -- (0,0,0.7);
        \node[above right] at (1.2,-.075,0) {$x_2$};
        \node[right] at (0,1.2,0) {$x_1$};
        \node[above] at (0,0,.7) {$y$};
        \coordinate (LL) at (0,0,0);
        \coordinate (UL) at (1,0,0);
        \coordinate (LU) at (0,1,0);
        \coordinate (UU) at (1,1,0.5);
        \coordinate (UM) at (1,0.5,0);
        \coordinate (MU) at (0.5,1,0);

        \def\gridlevel{-0.8}
        \def\arrowlevel{-0.9}
        \def\arrowuplevel{-0.7}
        \draw (0,0,\gridlevel) -- (0,1,\gridlevel) -- (1,1,\gridlevel) -- (1,0,\gridlevel) -- cycle;
         \node[below,white] at (0,0.5,\arrowlevel) {$e(1)$};
         \node[right,white] at (0.5,1,\arrowlevel) {$e(2)$};

        \draw[line width=1.5pt,line join=bevel] (0,0,\gridlevel) -- (1,0,\gridlevel) -- (1,1,\gridlevel)  -- cycle;
        \draw[dotted,line cap=round] (0,0,\gridlevel) -- (0,0,0);
        \draw[dotted,line cap=round] (1,0,\gridlevel) -- (1,0,0);
        \draw[dotted,line cap=round] (1,1,\gridlevel) -- (1,1,0.5);
        \draw[line cap=round,-Latex] (0,0,\arrowuplevel) -- (1,0,\arrowuplevel);
        \draw[line cap=round,-Latex] (1,0,\arrowuplevel) -- (1,1,\arrowuplevel);
        \node[left] at (0.85,0,\arrowuplevel) {$e(2)$};
        \node[above] at (1,0.5,\arrowuplevel) {$e(1)$};
        \node[yscale=0.8,xslant=0.85] (A) at (0.9,0.1,\arrowlevel) {  $\mathbf{T}_{2}$};
          \draw [fill=gray!80,opacity=0.7] (LL) -- (UL) -- (UM) -- (MU) -- (LU) -- cycle;
        \draw [fill=gray!80,opacity=0.7] (UU) -- (UM) -- (MU) -- cycle;
         \draw[fill=gray!80,opacity=0.5,line width=0.5pt] (LL) -- (UL) -- (UU) -- cycle;
    \end{tikzpicture}
     \captionsetup{font=small}
        \caption{Constructing $r_2$ using $\mathbf{T}_{2}$.  \label{fig:relunew:b}}
      \end{subfigure}
 \hspace{0.2in}
          \begin{subfigure}{.3\linewidth}
     \begin{tikzpicture}[yzx]
        \draw [->, >=stealth, line width=0.25] (0,0,0) -- (1.2,0,0);
        \draw [->, >=stealth, line width=0.25] (0,0,0) -- (0,1.2,0);
        \draw [->, >=stealth, line width=0.25] (0,0,0) -- (0,0,0.7);
        \node[above right] at (1.2,-.075,0) {$x_2$};
        \node[right] at (0,1.2,0) {$x_1$};
        \node[above] at (0,0,.7) {$y$};
        \coordinate (LL) at (0,0,0);
        \coordinate (UL) at (1,0,0);
        \coordinate (LU) at (0,1,0);
        \coordinate (UU) at (1,1,0.5);
        \coordinate (UM) at (1,0.5,0);
        \coordinate (MU) at (0.5,1,0);

        \def\gridlevel{-0.8}
        \def\arrowlevel{-0.9}
        \def\arrowuplevel{-0.7}

        \def\sepxone{0.6}
        \def\sepxtwo{0.3}
         \node [
    fill,
    red,
    circle,
    inner sep=1.5pt
] (tangent)  at (\sepxtwo,\sepxone,\gridlevel) {};
         \draw[dotted,line cap=round] (\sepxtwo,\sepxone,\sepxtwo*0.5) -- (\sepxtwo,\sepxone,\gridlevel);

        \draw [fill=gray!80,opacity=0.7] (LL) -- (UL) -- (UM) -- (MU) -- (LU) -- cycle;
        \draw [fill=gray!80,opacity=0.7] (UU) -- (UM) -- (MU) -- cycle;
        \draw (0,0,\gridlevel) -- (0,1,\gridlevel) -- (1,1,\gridlevel) -- (1,0,\gridlevel) -- cycle;
        \draw[fill=gray!80,opacity=0.5,line width=0.5pt]  (LL) -- (LU) -- (UU) -- cycle;
         \node [
    fill,
    red,
    diamond,
    inner sep=1.5pt
] (tangent)  at (\sepxtwo,\sepxone,\sepxtwo*0.5) {};
    \draw[dotted,line cap=round] (\sepxtwo,\sepxone,\sepxtwo*0.5) -- (\sepxtwo,\sepxone,0.5);
    \node [
    fill,
    red,
    star,star points=9,,
    inner sep=1.5pt
] (tangent)  at (\sepxtwo,\sepxone,0.5) {};
        \draw[line width=1.5pt,line join=bevel] (0,0,\gridlevel) -- (0,1,\gridlevel) -- (1,1,\gridlevel)  -- cycle;
        \draw[dotted,line cap=round] (0,0,\gridlevel) -- (0,0,0);
        \draw[dotted,line cap=round] (0,1,\gridlevel) -- (0,1,0);
        \draw[dotted,line cap=round] (1,1,\gridlevel) -- (1,1,0.5);
        \draw[line cap=round,-Latex] (0,0,\arrowlevel) -- (0,1,\arrowlevel);
        \draw[line cap=round,-Latex] (0,1,\arrowlevel) -- (1.2,1,\arrowlevel);
        \node[below] at (0,0.5,\arrowlevel) {$e(1)$};
        \node[right] at (0.5,1,\arrowlevel) {$e(2)$};
        \node[yscale=0.8,xslant=0.85] (A) at (0.6,0.67,\arrowlevel) {  $\mathbf{T}_{1}$};


    \end{tikzpicture}
     \captionsetup{font=small}
        \caption{Checking membership in $\Conv(S)$.  \label{fig:relunew:c}}
        \end{subfigure}
    \caption{Using interpolation on triangles to construct $\Conv(S)$ for \autoref{simpleexample}.
    }
    \label{fig:relunew}
\end{figure}

The subdivision of $[0,1]^2$ into  $\mathbf{T}_{1}$ and  $\mathbf{T}_{2}$ can be extended to $[0,1]^n$ by considering all $n!$ possible orders in which we can obtain $\bfone^n$ from $\bfzero^n$ by incrementally adding the canonical vectors. We represent these orders using the set of all permutations of $\sidx{n}$. In Example~\ref{simpleexample}, this set is given by $\mathcal{S}_2\defeq \Set{\pi_1,\pi_2}$, where  $\pi_i:\sidx{2}\to\sidx{2}$ for each $i\in\sidx{2}$, $\pi_1(1)=1$, $\pi_1(2)=2$, $\pi_2(1)=2$, and $\pi_2(2)=1$. Then,  under the notation of \autoref{kuhntriand:def} below, we have $\mathbf{T}_{1}=\mathbf{T}_{\pi_1}$ and $\mathbf{T}_{2}=\mathbf{T}_{\pi_2}$.
\begin{definition}\label{kuhntriand:def}
Let $\mathcal{S}_n$ be the set of all permutations of $\sidx{n}$. Then for every $\pi\in\mathcal{S}_n$, we define $\mathbf{V}_{\pi}=\Set{\sum_{i=0}^j e\rbra{\pi\rbra{i}}}_{j=0}^n$ and
\begin{equation}\label{triangle}
\mathbf{T}_{\pi}=\conv\rbra{\mathbf{V}_{\pi}}=\Set{x\in \bbR^n| 1\geq x_{\pi(1)}\geq x_{\pi(2)}\geq \ldots\geq x_{\pi(n)}\geq 0 }.
\end{equation}
\end{definition}
The collection of simplices $\set{\mathbf{T}_{\pi}}_{\pi \in \mathcal{S}_n}$, whose union is $[0,1]^n$, is known as the \emph{Kuhn triangulation} of $[0,1]^n$ \cite{todd1976computation}.

The number of simplices in the Kuhn triangulation is exponential, so an $n$-dimensional generalization of \autoref{simpleexample} could contain an exponential number of inequalities in \eqref{relunew:eq2}. Fortunately, as illustrated in the following example, the characterization of $\mathbf{T}_{\pi}$ in the right hand side of \eqref{triangle} allow us to easily filter for relevant inequalities.

\begin{namedexample}[\autoref{simpleexample} continued]
Consider the point $\rbra{x^*,y^*}=(0.6,0.3,0.5)$ depicted as a red star in  \autoref{fig:relunew:c}. To check if $\rbra{x^*,y^*}\in \Conv(S)$ we can first verify that $y^*\geq g\rbra{x^*}$ and $x^*\in [0,1]^2$. It then only remains to check that $\rbra{x^*,y^*}$ satisfies all inequalities in \eqref{relunew:eq2}. However, we can instead exploit the fact that if $x\in \mathbf{T}_{1}$, then $r_1(x)=\min\Set{r_1(x),r_2(x)}$. As illustrated in  \autoref{fig:relunew:c} we can use the fact that $x_1^* \geq x_2^*$ to conclude that $x^*$ (depicted as a red circle in \autoref{fig:relunew:c}) belongs to $ \mathbf{T}_{1}$. Finally, we can check that  $r_1\rbra{x^*}=0.3<0.5$ to conclude that $\rbra{x^*,y^*}\notin \Conv(S)$ (Point $\rbra{x^*,0.3}$ is depicted as a red diamond in \autoref{fig:relunew:c}).
\end{namedexample}

To show that the ideas in \autoref{simpleexample} can be generalized, we will exploit properties of submodular functions. For that we connect functions from $[0,1]^n$ with set-functions. We pick one specific connection that simplifies the statement and proof of \autoref{thm:convex-hull}.
\begin{definition} \label{def:H}
A set-function $H:2^{\llbracket n \rrbracket}\to\bbR$ is submodular if
\[H(S)+H(T)\geq H(S\cup T) + H(S\cap T)\quad \forall S, T\subseteq \sidx{n}.\]
For any $h:[0,1]^n\to\bbR$ we define the set-function  $H:2^{\llbracket n \rrbracket}\to\bbR$  given by  $H(I)=h\rbra{\sum_{i\notin I} e(i)}$ for each $I\subseteq \sidx{n}$. In particular, $H\rbra{\sidx{n}}=h\rbra{\bfzero^n}$ and $H\rbra{\emptyset}=h\rbra{\bfone^n}$. In general, for any function from $[0,1]^n$ to $\bbR$ defined as a lower case letter (e.g. $h$), we let the associated set-function be defined by the upper case version of this letter (e.g. $H$).
\end{definition}


\subsubsection{Proof of \autoref{thm:convex-hull}}

Our proof has three steps. First, we formalize the idea in Example~\ref{simpleexample} for arbitrary dimensions (Theorem~\ref{thm:convex-hull:alt:simple}). Then, we reduce the number of inequalities by characterizing which of the simplices $\mathbf{T}_{\pi}$ lead to identical inequalities (Lemma~\ref{lem:convex-hull:alt:interpolation}). Finally, to complete the proof of Theorem~\ref{thm:convex-hull}, we describe the explicit form of these inequalities.

Corollary 3.14 in \cite{tawarmalani2013explicit} gives us a precise description of $\Conv\rbra{Q}$ where $Q$ is a normalized version of $S$ from \eqref{Sdef:eq} that also considers any convex activation function. We include a submodularity-based proof of the corollary for completeness, adapted to our context.

\begin{restatable}{theorem}{convsimple} \label{thm:convex-hull:alt:simple}
Let $w\in \bbR^n_+$ and $b\in \bbR$,
 $f(x)=w\cdot x+b$, $\rho:\bbR\to\bbR$ be any convex function, $g(x)=\rho(f(x))$ and
 $Q = \Set{(x,y)\in[0,1]^n\times \bbR |  y = \rho(f(x))}$.


For each $\pi\in \mathcal{S}_n$ let $r_{\pi}:[0,1]^n \to \bbR$ be the unique affine interpolation\footnote{
Such an affine interpolation exists and is unique because $\mathbf{V}_{\pi}$ is a set of $n+1$ affinely independent points.
} of $g$ on $\mathbf{T}_{\pi}$ such that $r_{\pi}\rbra{v}=g(v)$ for all $v\in\mathbf{V}_{\pi}$.
Then $\Conv(Q)$ equals the set of all $(x,y)\in\bbR^n\times\bbR$  satisfying
\begin{subequations} \label{eqn:relu-conv-hull:simple}
\begin{align}
y &\geq g(x)\label{eqn:lb1:simple} \\
y &\leq r_{\pi}\rbra{x} \quad &\forall \pi \in \mathcal{S}_n \label{newfacets} \\
0 &\leq x_i \leq 1 \quad &\forall i \in \llbracket n \rrbracket
\end{align}
\end{subequations}

\end{restatable}


\begin{proof}
Let $h:[0,1]^n\to \bbR$ be such that $h(x)=-g(x)=-\rho(f(x))$ for all $x\in [0,1]^n$. In addition, let $\underline{h}$ and $\overline{h}$  respectively be the convex and concave envelopes of $h$ (i.e.\ the largest convex underestimator of $h$, which is well-defined because the pointwise maximum of convex functions lying below $h$ is a convex function, and the smallest concave overestimator of $h$, which is similarly well-defined). Then $Q = \Set{(x,y)\in[0,1]^n\times \bbR |  - y = h(x)}$ and $\Conv\rbra{Q}=\Set{(x,y)\in[0,1]^n\times \bbR |  \underline{h}(x) \leq -y \leq \overline{h}(x)}$ (e.g. \cite[Proposition B.1]{salman2019convex}). Function $h$ is concave and hence $\overline{h}=h$, so it only remains to describe $\underline{h}$.

To describe $\underline{h}$, we define a set function $H$ based on $h$ (see Definition~\ref{def:H}), which is submodular
because $-\rho$ is concave and $w$ is non-negative (e.g. see \cite[Section 3.1]{ahmed2011maximizing}).
Submodularity allows us to describe the lower convex envelope of the continuous function $h$ through the \emph{Lov{\'a}sz extension} of the set function $H$.
This extension is the piecewise affine function from $[0,1]^n$ to $\bbR$ defined over the pieces $\{\mathbf{T}_{\pi}:\pi\in\mathcal{S}_n\}$, which equals $\max_{\pi\in\mathcal{S}_n}(-r_{\pi})$ by convexity (e.g. see \cite{bach2013learning} for further details).
Therefore the constraint required for $\conv(Q)$ is $\underline{h}(x)\le-y\Longleftrightarrow y\le\min_{\pi\in\mathcal{S}_n}r_{\pi}(x)$ which completes the derivation of inequalities~\eqref{newfacets} in the theorem statement.
\end{proof}

Note that even though there are exponentially many inequalities in~\eqref{newfacets},
the tightest constraint on $y$ at any given point $x\in[0,1]^n$ can be efficiently found,
by sorting the coordinates of $x$ to find the simplex $\mathbf{T}_{\pi}$ to which $x$ belongs.
Moreover,
going from $[0,1]^n$ to $[L,U]$ and eliminating the sign restriction on $w$ can be achieved with standard variable transformations (e.g. see the comments before \cite[Corollary 3.14]{tawarmalani2013explicit}).

Before demonstrating the variable transformations,
we first further refine \autoref{thm:convex-hull:alt:simple} for the case when $\rho$ is equal to the ReLU activation function $\sigma$. In particular, we generally have that each one of the $n!$ inequalities in \eqref{newfacets} is  facet-defining because they hold at equality over the $n+1$ affinely independent points  $\set{\rbra{v,g(v)}}_{v\in\mathbf{V}_{\pi}}$. Hence, they are all needed to describe $\Conv(R)$. However, because it may happen that $r_{\pi}=r_{\pi'}$ for $\pi\neq \pi'$, the number of inequalities in \eqref{newfacets} after removing duplicates may be much smaller. The following lemma shows that this is indeed the case when $\rho$ is equal to the ReLU activation function $\sigma$. The lemma also gives a closed form expression for the interpolating functions $r_{\pi}$ in this case.



\begin{restatable}{lemma}{interpolationlemma}\label{lem:convex-hull:alt:interpolation}
Let $w\in \bbR^n_+$ and $-\sum_{i=1}^n w_i\leq b<0$, $f(x)=w\cdot x+b$, and $g(x)=\sigma(f(x))$.   If
 $\Set{r_{\pi}}_{\pi\in \mathcal{S}_n}$ are  the  affine interpolation functions from \autoref{thm:convex-hull:alt:simple}, then
\[
\Set{\rbra{x,y}\in \bbR^{n+1} | y\leq {r}_{\pi}\rbra{x} \quad \forall \pi \in\mathcal{S}_n}=\Set{\rbra{{x},y}\in \bbR^{n+1} | y\leq {r}_{I,h}\rbra{{x}} \quad \forall \rbra{I,h}\in {\calI}}\]
where
$\calI\defeq\Set{
        (I,h) \in 2^{\llbracket n \rrbracket} \times \llbracket n \rrbracket |
                    F(I) \geq 0,\  F(I\cup\{ h\}) < 0
    }$,
    ${r}_{I,h}\rbra{{x}}\defeq F(I) {x_h} + \sum_{i\in I} {w}_{i} {x}_{i}$, and $F:2^{\sidx{n}}\to \bbR$ is the set-function associated to $f$ as  defined in Definition~\ref{def:H}.

\end{restatable}

 \begin{proof}
Fix  $\pi \in\mathcal{S}_n$ and for each $j\in \sidx{n}$ let $I(j)\defeq\Set{\pi(i)}_{i=j+1}^n$.


Then the interpolation condition for ${r}_{\pi}$ given by  $r_{\pi}\rbra{v}=g(v)$ for all $v\in\mathbf{V}_{\pi}$ is equivalent to
\begin{equation}\label{lem:convex-hull:alt:interpolation:eq:int}
    R_{\pi}\rbra{I(j)}=G\rbra{I(j)}\quad \forall j=0,1,\ldots,n
\end{equation}
where $R_{\pi}$ and $G$ are the set-functions associated to $r_{\pi}$ and $g$ as defined in \autoref{def:H}.
For $j=0$, condition \eqref{lem:convex-hull:alt:interpolation:eq:int} implies $r_{\pi}\rbra{\bfzero^n}=R_{\pi}\rbra{\sidx{n}}=G\rbra{\sidx{n}}=g(\bfzero^n)=0$ and hence there exists $\alpha\in \bbR^n$ such that ${r}_{\pi}\rbra{{x}}=\alpha \cdot {x}$ (i.e. ${r}_{\pi}$ is a linear function). For $j\in\sidx{n}$, condition \eqref{lem:convex-hull:alt:interpolation:eq:int} further implies that
\begin{equation}\label{lem:convex-hull:alt:interpolation:eq:alphas}
\alpha_{\pi(j)} = g\rbra{\sum\nolimits_{i=0}^j e\rbra{\pi\rbra{i}}}- g\rbra{\sum\nolimits_{i=0}^{j-1} e\rbra{\pi\rbra{i}}}=G\rbra{I(j)}-G\rbra{I(j-1)} \quad \forall j\in \sidx{n}.
\end{equation}
Now, because $F(\emptyset)=f(\bfone^n)\geq 0$, $w\in \bbR^n_+$, and $b<0$, there exists a unique  $k\in \sidx{n}$ such that $\rbra{I(k),\pi(k)}\in \calI$. Furthermore,   $w\in \bbR^n_+$, $F\rbra{I\rbra{k}\cup\Set{\pi\rbra{k}}}=F\rbra{I\rbra{k-1}}<0$, and $F\rbra{I\rbra{k}}\geq 0$ imply
\begin{subequations}\label{lem:convex-hull:alt:interpolation:eq:alphas12}
\begin{alignat}{3}
F\rbra{I(j)}&<0\quad &\text{ and }\quad G\rbra{I(j)}&=0 &\quad &\forall j=0,\ldots,k-1; \label{lem:convex-hull:alt:interpolation:eq:alphas1}\\
F\rbra{I(j)}&\geq 0\quad &\text{ and }\quad G\rbra{I(j)}&=F\rbra{I(j)} &\quad &\forall j=k,\ldots,n. \label{lem:convex-hull:alt:interpolation:eq:alphas2}
\end{alignat}
\end{subequations}
Equations \eqref{lem:convex-hull:alt:interpolation:eq:alphas} and \eqref{lem:convex-hull:alt:interpolation:eq:alphas1} imply $\alpha_{\pi(j)}=0$ for all $j\in \sidx{k}$ or equivalently $\alpha_i=0$ for all $i\notin I\rbra{k}\cup\Set{\pi\rbra{k}}$. Equations \eqref{lem:convex-hull:alt:interpolation:eq:alphas} and \eqref{lem:convex-hull:alt:interpolation:eq:alphas12} imply $\alpha_{\pi\rbra{k}}= G\rbra{I\rbra{k}}=F\rbra{I\rbra{k}}$. Finally, equations \eqref{lem:convex-hull:alt:interpolation:eq:alphas} and \eqref{lem:convex-hull:alt:interpolation:eq:alphas2} imply that $\alpha_{\pi(j)}=w_{\pi\rbra{j}}$ for all $j=k+1,\ldots,n$ or equivalently $\alpha_i=w_i$ for all $i\in I$. Hence, $r_\pi=r_{I\rbra{k},\pi\rbra{k}}$. The lemma follows by noting that for any $\rbra{I,h}\in {\calI}$ there exists at least one $\pi \in\mathcal{S}_n$ such that $\rbra{I\rbra{k},\pi\rbra{k}}=\rbra{I,h}$.
\end{proof}



Finally, we obtain the proof of Theorem~\ref{thm:convex-hull} recalling that $f(x)= w\cdot x+b$ for $w\in \bbR^n$ and $b\in \bbR$, and $S = \Set{(x,y)\in [L,U]\times \bbR |  y = \sigma(f(x)) }$ for  $L,U\in\bbR^n$  such that $L<U$.

\generaltheo*
\begin{proof}
Recalling that $\calJ\defeq\Set{
        (I,h) \in 2^{\llbracket n \rrbracket} \times \llbracket n \rrbracket |
                    \ell(I) \geq  0,\quad
            \ell(I\cup\{ h\}) < 0,\quad w_i\neq 0\quad \forall i\in I}$ we can assume without loss of generality that $w_i\neq 0$ for all $i\in\sidx{n}$ and hence $d=n$ (Indices $i$ with $w_i=0$ do not affect \eqref{eqn:cut-family-2} or the definition of $\calJ$ and the only inequalities for $S$ or $\Conv(S)$ in which a given $x_i$ appears are $L_i\leq x_i\leq U_i$).

For the first case, the result follows because $f(x)<0$ for all $x\in[L,U]$ and hence $g(x)=0$ for all $x\in[L,U]$.

For the second case, the result follows because $f(x)\geq0$ for all $x\in[L,U]$ and hence $g(x)=f(x)$ for all $x\in[L,U]$.

For the third case, recall that  $\breve{L}_i = \begin{cases} L_i & w_i \geq 0 \\ U_i & \text{o.w.} \end{cases}$ and $\breve{U}_i = \begin{cases} U_i & w_i \geq 0 \\ L_i & \text{o.w.} \end{cases}$, and consider the affine variable transformation given by
\begin{equation}\label{thm:convex-hull:eq:conversion1}
    \breve{x}_i\defeq \frac{x_i-\breve{L}_i}{\breve{U}_i-\breve{L}_i} \quad \text{ and } \quad x_i= (\breve{U}_i-\breve{L}_i) \breve{x}_i +\breve{L}_i\quad \forall i\in \sidx{n}.
\end{equation}
Let  $\breve{w}_i\defeq w_i (\breve{U}_i-\breve{L}_i)$ for each $i\in \sidx{n}$,   $\breve{b}\defeq b+\sum_{i=1}^n w_i \breve{L}_i=\ell(\sidx{n})< 0$, and $\breve{f}(\breve{x})\defeq\breve{w}\cdot \breve{x}+\breve{b}$ (recall that  $\ell(I) \defeq \sum_{i \in I} w_i\breve{L}_i + \sum_{i \not\in I} w_i\breve{U}_i + b$). Then we may infer that
\begin{equation}\label{thm:convex-hull:eq:conversion2}
\breve{w}_i \breve{x}_i=w_i (x_i-\breve{L}_i)\quad \forall i\in \sidx{n},
\end{equation}
that $f(x)=\breve{f}(\breve{x})$, and finally that $\rbra{x,y}\in S$ if and only if $\rbra{\breve{x},y}\in \breve{S} \defeq \Set{\rbra{\breve{x},y}\in [0,1]^n\times \bbR |  y = \sigma(\breve{f}(\breve{x})) }$.

In addition, we conclude $\breve{w}\in \bbR^n_+$, using the definition of $\breve{L}$ and $\breve{U}$ and the fact that $L<U$. Hence, \autoref{thm:convex-hull:alt:simple} and \autoref{lem:convex-hull:alt:interpolation} are applicable for $\breve{S}$ and $\breve{g}(\breve{x})=\sigma(\breve{f}(\breve{x}))$. Then
\[\Conv(\breve{S})=\Set{\rbra{\breve{x},y}\in  [0,1]^n\times \bbR_+ | \breve{f}\rbra{\breve{x}}\leq y\leq {r}_{I,h}\rbra{\breve{x}} \quad \forall \rbra{I,h}\in {\calI}}\]
where $\calI=\Set{
        (I,h) \in 2^{\llbracket n \rrbracket} \times \llbracket n \rrbracket |
                    \breve{F}(I) \geq 0, \ \breve{F}(I\cup\{ h\}) < 0
    }$ and
    ${r}_{I,h}(\breve{x})= \breve{F}(I) {\breve{x}_h} + \sum_{i\in I} \breve{w}_{i} \breve{x}_{i}$. Using the definitions of $\breve{b}$ and $ \breve{w}_{i}$ we get
 \begin{equation}\label{thm:convex-hull:eq:lconversion}
     \breve{F}(I)= \sum_{i\notin I} \breve{w}_{i} +\breve{b} = \sum_{i\notin I} w_i \rbra{\breve{U}_i-\breve{L}_i} +\rbra{b+\sum_{i=1}^n w_i \breve{L}_i}= \sum_{i\notin I} w_i\breve{U}_i + \sum_{i\in I} w_i\breve{L}_i+b=\ell\rbra{I}
 \end{equation}
  and hence $\calI=\calJ=\Set{
        (I,h) \in 2^{\llbracket n \rrbracket} \times \llbracket n \rrbracket |
                    \ell(I) \geq  0,\:
            \ell(I\cup\{ h\}) < 0
    }$. Combining (\ref{thm:convex-hull:eq:conversion1}--\ref{thm:convex-hull:eq:lconversion}), we get
\[\breve{r}_{I,h}\rbra{\breve{x}}=\breve{F}\rbra{I} \breve{x}_{h} + \sum_{i\in I} \breve{w}_{i} \breve{x}_{i}= \ell\rbra{I} \frac{x_h-\breve{L}_h}{\breve{U}_h-\breve{L}_h} +\sum_{i\in I}w_i \rbra{x_i-\breve{L}_i}. \]
Hence, $\Conv\rbra{S}$ is described by \eqref{eqn:relu-conv-hull}.


Finally, $\rbra{I,h}\in\calJ$ if and only if the hyperplane   $\sum_{i=1}^n w_i x_i + b = 0$ cuts the edge $uv$ of $[0,1]^n$ given by  $u\defeq\sum_{i\notin \rbra{I\cup \set{h}}} e(i)$ and $v\defeq \sum_{i\notin I} e(i)$  (with the convention that an empty sum is equal to zero). The result on $\abs{\calJ}$ then follows by \autoref{hyperplanelemma} recalling that without loss of generality we have assumed $n=d$.

\end{proof}

\end{subsection}
\begin{subsection}{An alternative proof using mixed-integer programming and projection}\label{secondproof}

\newcommand{\Rsharp}{R_{\operatorname{sharp}}}
\newcommand{\Smax}{S_{\operatorname{max}}}






We can alternatively prove Theorem~\ref{thm:convex-hull} by connecting it to the MIP formulation from \cite{anderson2020strong} for  $S$ defined in \eqref{Sdef:eq}. For this, first recall that  that $f(x)= w\cdot x+b$ for $w\in \bbR^n$ and $b\in \bbR$, and $S = \Set{(x,y)\in [L,U]\times \bbR |  y = \sigma(f(x)) }$ for  $L,U\in\bbR^n$  such that $L<U$.

\begin{corollary}\label{MPApapercoro}
Let
\[
    \Rsharp \defeq \Set{(x,y,z)\in [L,U]\times \bbR\times [0,1]^2 |
    \begin{alignedat}{4}
    y &\geq 0,\\
    y &\geq w \cdot x + b,\\
    y &\leq \bar{f}(x, z),\\
    z_1+z_2&=1
    \end{alignedat}},
\]
where
\begin{align*}
    \bar{f}(x, z) \defeq \max_{\tx^1, \tx^2}\Set{w \cdot \tx^2 + b z_2 |
    \begin{alignedat}{3}
    x &= \tx^1+\tx^2,  \\
    L z_k \leq \tx^k &\leq U  z_k &\quad& \forall k \in \llbracket 2 \rrbracket \\
    \tx^1, \tx^2&\in \bbR^n
    \end{alignedat}}.
\end{align*}
Then $\Conv\rbra{S}=\Proj_{x,y}\rbra{\Rsharp}\defeq \Set{(x,y)\in \bbR^{n+1}|\exists z\in \bbR^2\:\text{ s.t. } \: \rbra{x,y,x}\in \Rsharp}$.
\end{corollary}
\begin{proof}
Follows from \cite[Proposition 5]{anderson2020strong} for the case $d=2$, $w^1=0$, $b^1=0$, $w^2=w$, $b^2=b$.
\end{proof}
\begin{lemma}\label{projectedMPlemma}
Let
\[
    R \defeq \Set{(x,y)\in [L,U]\times \bbR |
    \begin{alignedat}{4}
    y &\geq 0,\\
    y &\geq w \cdot x + b,\\
    y &\leq \tilde{f}(x)\\
    \end{alignedat}}
\]
where
\begin{equation}\label{tildahdef}
    \tilde{f}(x) \defeq \max_{\tx^1, \tx^2,z}\Set{w \cdot \tx^2 + b z_2 |
    \begin{alignedat}{3}
    x &= \tx^1+\tx^2,  \\
    L z_k \leq \tx^k &\leq U  z_k &\quad& \forall k \in \llbracket 2 \rrbracket \\
     \tx^1, \tx^2&\in \bbR^n\\
    z_1+z_2&=1\\
    z&\in [0,1]^2
    \end{alignedat}}.
    \end{equation}
Then $\Conv\rbra{S}=R$.
\end{lemma}
\begin{proof}
By \autoref{MPApapercoro} it suffices to show $R = \Proj_{x,y}(\Rsharp)$.

Inclusion $\Proj_{x,y}(\Rsharp) \subseteq R$ follows by noting that $\bar{f}(\hat{x}, \hat{z}) \leq \tilde{f}(\hat{x})$  for any $(\hat{x}, \hat{y}, \hat{z}) \in \Rsharp$.

For inclusion $R \subseteq \Proj_{x,y}(\Rsharp)$, let $(\hat{x}, \hat{y}) \in R$, and let $\rbra{\tx^1, \tx^2,z}\in \bbR^{2n+2}$ be an optimal solution to the optimization problem in the right hand side of \eqref{tildahdef} for $x=\hat{x}$. Such solution exists because for $\hat{x}\in [L,U]$ this optimization problem is the maximization of a linear function over a non-empty bounded polyhedron. Then, $\tilde{f}(\hat{x}) = \bar{f}(\hat{x}, {z})$, and hence $(\hat{x}, \hat{y}, z) \in \Rsharp$.
\end{proof}

\generaltheo*
\begin{proof}

Recalling that $\calJ\defeq\Set{
        (I,h) \in 2^{\llbracket n \rrbracket} \times \llbracket n \rrbracket |
                    \ell(I) \geq  0,\quad
            \ell(I\cup\{ h\}) < 0,\quad w_i\neq 0\quad \forall i\in I}$ we can assume without loss of generality that $w_i\neq 0$ for all $i\in\sidx{n}$ and hence $d=n$ (Indices $i$ with $w_i=0$ do not affect \eqref{eqn:cut-family-2} or the definition of $\calJ$ and the only inequalities for $S$ or $\Conv(S)$ in which $x_i$ appear for such index are $L_i\leq x_i\leq U_i$).

For the first case, the result follows because $f(x)<0$ for all $x\in[L,U]$ and hence $g(x)=0$ for all $x\in[L,U]$.

For the second case, the result follows because $f(x)\geq0$ for all $x\in[L,U]$ and hence $g(x)=f(x)$ for all $x\in[L,U]$.

For the third case,  it suffices to show that
\begin{equation}\label{alternateprooffunceq}
    \tilde{f}(x) = \min_{(I, h) \in \calJ} \left\{\sum_{i \in I} w_i(x_i-\breve{L}_i) + \frac{\ell(I)}{\breve{U}_{ h}-\breve{L}_{ h}} (x_{ h}-\breve{L}_{ h})\right\},
\end{equation}
in which case, set $R$ from \autoref{projectedMPlemma} is exactly the set described by  \eqref{eqn:relu-conv-hull}. To show \eqref{alternateprooffunceq} we first simplify the optimization problem defining  $\tilde{f}(x)$ by applying the simple substitutions $\tilde{x} \defeq \tilde{x}^2 = x - \tilde{x}^1$ and $z \defeq z_2 = 1 - z_1$:
\begin{align*}
    \tilde{f}(x) = \max_{\tilde{x}, z}\Set{w \cdot \tilde{x} + b z\ |
    \begin{alignedat}{3}
    L(1-z) &\leq\ & x - \tx &\leq U(1-z),\\ Lz &\leq\ & \tx &\leq Uz,\\ z &\in& [0,1]
    \end{alignedat}
    }.
\end{align*}

This optimization problem is feasible and bounded when $L \leq x \leq U$, and thus we may assume an optimal solution exists.


Consider some $i \in \llbracket n \rrbracket$. If $w_i > 0$, then $\tilde{x}_i \geq L_i z$ and $x_i - \tilde{x}_i \leq U_i (1 - z)$ hold at any optimal solution, since we are maximizing the problem and each constraint involves only a single $x_i$ and $z$. Analogously, if $w_i < 0$, then $\tilde{x}_i \leq U_i z$ and $x_i - \tilde{x}_i \geq L_i (1 - z)$ are implied as well. To unify these two cases into one as a simplification, observe that these constraints can be expressed as $w_i \tilde{x}_i \geq w_i\breve{L}_i z$ and $w_i (x_i - \tilde{x}_i) \leq w_i \breve{U}_i (1 - z)$ respectively (recall that $w_i \neq 0$ by assumption, and that $\breve{L}_i = L_i$ if $w_i \geq 0$, or $U_i$ otherwise, and $\breve{U}_i = U_i$ if $w_i \geq 0$, or $L_i$ otherwise). Therefore, we can drop these constraints and keep the remaining ones:
\begin{align*}
    \tilde{f}(x) = \max_{\tilde{x}, z}\Set{w \cdot \tilde{x} + b z\ |
    \begin{alignedat}{3}
    w_i (x_i - \tx_i) &\geq w_i \breve{L}_i(1-z) &\quad& \forall i \in \llbracket n \rrbracket,\\
    w_i \tx_i &\leq w_i \breve{U}_iz &\quad& \forall i \in \llbracket n \rrbracket \\
    z &\in [0,1]
    \end{alignedat}
    }.
\end{align*}

Define $\gamma_i \defeq w_i (\breve{U}_i z - \tilde{x}_i)$ for all $i \in \llbracket n \rrbracket$. We can then rewrite the problem as:
\begin{align*}
    \tilde{f}(x) = \max_{\gamma, z}\Set{(w \cdot \breve{U} + b) z - \sum_{i=1}^n \gamma_i \ |
    \begin{alignedat}{3}
    w_i (\breve{U}_i - \breve{L}_i)z - \gamma_i &\leq w_i (x_i - \breve{L}_i) &\quad&\forall i \in \llbracket n \rrbracket \\
    \gamma &\geq 0,\\ z &\in [0,1]
    \end{alignedat}
    }.
\end{align*}

We next take the dual of this problem. By strong duality, the following holds:
\begin{align*}
    \tilde{f}(x) = \min_{\alpha, \beta}\Set{\sum_{i=1}^n w_i (x_i - \breve{L}_i) \alpha_i + \beta \ |
    \begin{alignedat}{3}
    \sum_{i=1}^n w_i (\breve{U}_i - \breve{L}_i) \alpha_i + \beta &\geq \sum_{i=1}^n w_i \breve{U}_i + b,\\ \alpha &\in [0,1]^n,\\ \beta &\geq 0
    \end{alignedat}
    }.
\end{align*}

To conclude the proof, we describe the optimal solutions of the optimization problem above.
Note that it is a minimization variant of a fractional knapsack problem and it can be solved by a greedy algorithm, in which we order the indices of $\alpha$ by $\frac{x_i - \breve{L}_i}{\breve{U}_i - \breve{L}_i}$ and maximally select those with the smallest ratios, until the knapsack constraint is satisfied at equality. We also need to consider $\beta$ in the knapsack, but since the ratios for $\alpha_i$ are in $[0,1]$ and the ratio for $\beta$ is 1, $\beta$ would only be picked last. Moreover, under the assumptions of our current third case, we have $\ell(\sidx{n})=\sum_{i=1}^n w_i \breve{L}_i + b < 0$, and thus that we can satisfy the knapsack constraint by choosing from $\alpha$'s (recall that $\ell(I) = \sum_{i \in I} w_i\breve{L}_i + \sum_{i \not\in I} w_i\breve{U}_i + b$). Therefore we may set $\beta = 0$.


Let $I$ be the set of indices in which $\alpha_i = 1$ for the optimal solution and $h$ be the next index to be considered by the greedy procedure after the elements in $I$. Then
\[ \alpha_h = \frac{\rbra{\sum_{i=1}^n w_i \breve{U}_i + b}-\rbra{\sum_{i\in I} w_i (\breve{U}_i - \breve{L}_i)}}{\breve{U}_h - \breve{L}_h}=\frac{\ell(I)}{\breve{U}_h - \breve{L}_h}\in [0,1).\]

Observe that $\ell(I) \geq 0$ is equivalent to stating that the items in $I$ are below the knapsack capacity, since $\ell(I)$ equals the capacity of the knapsack minus the total weight of the items in $I$. Therefore, $\ell(I) \geq 0$ and $\ell(I\cup\{ h\}) < 0$ (i.e.\ the items in $I$ fit but we can only add $h$ partially). Hence, we can write the optimization problem defining $\tilde{f}(x)$ as finding the optimal $I$ and $h$:
\begin{align*}
    \tilde{f}(x) = \min_{I, h \notin I} \left\{ \sum_{i \in I} w_i(x_i-\breve{L}_i) + \frac{\ell(I)}{\breve{U}_{ h}-\breve{L}_{ h}} (x_{ h}-\breve{L}_{ h})
    \ |\ \ell(I) \geq  0,\:
            \ell(I\cup\{ h\}) < 0\right\}.
\end{align*}
We obtain \eqref{alternateprooffunceq} by recalling that $\calJ=\Set{
        (I,h) \in 2^{\llbracket n \rrbracket} \times \llbracket n \rrbracket |
                    \ell(I) \geq  0,\:
            \ell(I\cup\{ h\}) < 0
    }$.

Finally, $\rbra{I,h}\in\calJ$ if and only if the hyperplane   $\sum_{i=1}^n w_i x_i + b = 0$ cuts the edge $uv$ of $[0,1]^n$ given by  $u\defeq\sum_{i\notin \rbra{I\cup \set{h}}} e(i)$ and $v\defeq \sum_{i\notin I} e(i)$  (with the convention that an empty sum is equal to zero). The result on $\abs{\calJ}$ then follows by \autoref{hyperplanelemma} recalling that without loss of generality we have assumed $n=d$.
\end{proof}

\end{subsection}
\end{section}

\begin{section}{Proofs of other results from Section~\ref{exactconvexsec}}\label{otherproofsfromthree}
\arbitrarygap*
  \begin{proof}
    This follows as a straightforward extension of \cite[Example 2]{anderson2020strong}, as the $\Delta$-relaxation is equal to the projection of the big-$M$ formulation presented in that work.
  \end{proof}

The following proposition shows how the additional structure in \autoref{lem:convex-hull:alt:interpolation} allows increasing the speed of checking for violated inequalities from $\mathcal{O}(n\log(n))$, achievable by sorting the input components, to $\mathcal{O}(n)$.

\separationlemma*
\begin{proof}
Recall that $\calJ\defeq\Set{
        (I,h) \in 2^{\llbracket n \rrbracket} \times \llbracket n \rrbracket |
                    \ell(I) \geq  0,\quad
            \ell(I\cup\{ h\}) < 0,\quad w_i\neq 0 \ \forall i\in I
    }$,
    $\ell(I) \defeq \sum_{i \in I} w_i\breve{L}_i + \sum_{i \not\in I} w_i\breve{U}_i + b$,
$\breve{L}_i \defeq \begin{cases} L_i & w_i \geq 0 \\ U_i & \text{o.w.} \end{cases}$ and $\breve{U}_i \defeq \begin{cases} U_i & w_i \geq 0 \\ L_i & \text{o.w.} \end{cases}$ for each $i\in\sidx{n}$, and \eqref{separationproblem} is the optimization problem given by
\begin{equation*}
    \upsilon(x)\defeq\min\Set{ \sum\nolimits_{i \in I} w_i(x_i-\breve{L}_i) + \frac{\ell(I)}{\breve{U}_{ h}-\breve{L}_{ h}} (x_{ h}-\breve{L}_{ h})  | (I, h) \in \calJ }.
\end{equation*}

First, we can check in $\mathcal{O}(n)$ time if $\ell(\sidx{n})\geq 0$ or $\ell(\emptyset)< 0$, in which case $\calJ=\emptyset$ and  \eqref{separationproblem} is infeasible.    Otherwise,  $\ell(\sidx{n})<0$, $\ell(\emptyset)\geq 0$, and $\calJ\neq\emptyset$.

We can also remove  in $\mathcal{O}(n)$ time all $i\in \sidx{n}$ such that $w_i=0$. Then without loss  of generality we may assume that $w_i\neq 0$ for all $i\in\sidx{n}$ and hence $L<U$ implies that
\begin{equation}\label{eq:separation-frac-knap-new-positiveweight}
    w_i (\breve{U}_i - \breve{L}_i)>0 \quad \forall i\in \sidx{n}.
\end{equation}

We will show that \eqref{separationproblem} is equivalent to the  linear programming problem
\begin{subequations} \label{eqn:separation-frac-knap-new-again}
\begin{align}
   \omega(x)\defeq \min_{v} \quad& \sum_{i=1}^n w_i(x_i-\breve{L}_i)v_i \\
    \text{s.t.}\quad& \sum_{i=1}^n w_i (\breve{U}_i - \breve{L}_i) v_i = \sum_{i=1}^n w_i \breve{U}_i + b,  \label{eqn:separation-frac-knap-eq}\\
    & 0 \leq v \leq 1.
\end{align}
\end{subequations}
Note that the set of basic feasible solutions for the linear programming problem is exactly the set of all feasible points with at most one fractional component (see, e.g., \cite[Chapter 3]{Bertsimas:1997}). That is, all basic feasible solutions of \eqref{eqn:separation-frac-knap-new-again} are elements of  $\calV\defeq \Set{v\in [0,1]^n | \abs{\Set{i\in \sidx{n}|v_i\in (0,1)}}\leq 1}$.

To prove that  $\omega(x)\leq  \upsilon(x)$,
consider the  mapping $\Phi:\calJ\to \calV$ given by
\[
    \Phi\rbra{\rbra{I,h}}_i = \begin{cases} 1 & i \in {I} \\ \frac{\ell({I})}{w_h(\breve{U}_h-\breve{L}_h)} & i = {h} \\ 0 & \text{o.w.} \end{cases} \quad \forall i \in \llbracket n \rrbracket.
\]
Let  $(\bar{I},\bar{h}) \in \calJ$ be an optimal solution for \eqref{separationproblem} and let $\bar{v}=\Phi\rbra{\rbra{\bar{I},\bar{h}}}$.
Then
\[\sum_{i=1}^n w_i (\breve{U}_i - \breve{L}_i) \bar{v}_i = \sum_{i\in \bar{I}} w_i (\breve{U}_i - \breve{L}_i) + w_{\bar{h}} (\breve{U}_{\bar{h}} - \breve{L}_{\bar{h}})\frac{\ell(\bar{I})}{w_{\bar{h}}(\breve{U}_{\bar{h}}-\breve{L}_{\bar{h}})}=\sum_{i=1}^n w_i \breve{U}_i + b,\]
 and hence $\bar{v}$ satisfies \eqref{eqn:separation-frac-knap-eq}.  Algebraic manipulation shows that
\begin{equation} \label{eqn:twiddle-one-comp-new}
    w_h(\breve{U}_h-\breve{L}_h) = \ell(I) - \ell(I \cup \{h\}) \quad \forall I\subseteq \sidx{n},\quad h\in \sidx{n}\setminus I.
\end{equation}
 In addition, $(\bar{I},\bar{h}) \in \calJ$ implies $\ell(\bar{I}) \geq 0$ and $\ell(\bar{I} \cup \{\bar{h}\}) < 0$. Combining this  with \eqref{eqn:twiddle-one-comp-new} gives the inequality $\ell(\bar{I}) < w_{\bar{h}}(\breve{U}_{\bar{h}}-\breve{L}_{\bar{h}})$. Therefore,  $\bar{v}_h \in [0, 1)$, and hence $\bar{v}_h$ is feasible for \eqref{eqn:separation-frac-knap-new-again}. In addition, for any  $(I,h)\in\calJ$ we have that
\[\sum_{i=1}^n w_i(x_i-\breve{L}_i)\bar{v}_i=\sum_{i\in {I}}  w_i(x_i-\breve{L}_i) + w_{{h}} (\breve{U}_{{h}} - \breve{L}_{{h}})\frac{\ell({I})}{w_h(x_h-\breve{L}_h)}\]
and hence  the objective value of  $\bar{v}_h$ for \eqref{eqn:separation-frac-knap-new-again} is the same as the objective value of $({I},{h})$ for \eqref{separationproblem}.

To prove $\omega(x)\geq  \upsilon(x)$ we will show that, through $\Phi$, the greedy procedure to solve \eqref{separationproblem}  described in the main text just before the statement of \autoref{lemma:general_separation}, becomes the standard greedy procedure for \eqref{eqn:separation-frac-knap-new-again} and hence also yields an optimal basic feasible solution to \eqref{eqn:separation-frac-knap-new-again}. For simplicity, assume without loss of generality that we have re-ordered the indices in $\sidx{n}$ so that
\begin{equation}\label{nondecreasingequation}
    \frac{x_1 - \breve{L}_1}{\breve{U}_1 - \breve{L}_1} \leq \frac{x_2 - \breve{L}_2}{\breve{U}_2 - \breve{L}_2} \leq \cdots \leq \frac{x_n - \breve{L}_n}{\breve{U}_n - \breve{L}_n}.
\end{equation}
Then the greedy procedure that incrementally grows $I$ terminates with some $(I,h)\in \calJ$ where $I=\sidx{h-1}$. Then $v= \Phi\rbra{\rbra{\sidx{h-1},h}} $ is a basic feasible solution for  \eqref{eqn:separation-frac-knap-new-again} with the same objective value as the objective value of $(I,h)$ for \eqref{separationproblem}. To conclude that $\omega(x)\geq  \upsilon(x)$, we claim that $v$ is an optimal solution for \eqref{eqn:separation-frac-knap-new-again} since the standard greedy procedure for \eqref{eqn:separation-frac-knap-new-again} is known to generate the optimal solution for this problem. For completeness, we give the following self contained proof of the claim. Assume for a contradiction that $\omega\rbra{x}<\sum_{i=1}^n w_i(x_i-\breve{L}_i)v_i$ and let  $v'$ be an optimal solution to  \eqref{eqn:separation-frac-knap-new-again}. Because $v'\neq v$ and both $v$ and $v'$ satisfy \eqref{eqn:separation-frac-knap-eq}, \eqref{eq:separation-frac-knap-new-positiveweight} implies there must exists $j_1,j_2\in \sidx{n}$ such that $j_1<j_2$, $j_1\leq h$, $v'_{j_1}<v_{j_1}$, $j_2\geq h$ and $v'_{j_2}>v_{j_2}$. Let $\epsilon>0$ be the largest value such that $v'_{j_1}+\frac{\epsilon}{w_{{j_1}} (\breve{U}_{{j_1}} - \breve{L}_{{j_1}})}\leq v_{j_1}$ and $v'_{j_2}-\frac{\epsilon}{w_{{j_2}} (\breve{U}_{{j_2}} - \breve{L}_{{j_2}})}\geq v_{j_2}$, and let
\[v''\defeq v'+\frac{\epsilon}{w_{{j_1}} (\breve{U}_{{j_1}} - \breve{L}_{{j_1}})}e(j_1)-\frac{\epsilon}{w_{{j_2}} (\breve{U}_{{j_2}} - \breve{L}_{{j_2}})}e(j_2).\]
By \eqref{nondecreasingequation} we either have
\begin{equation}
    \frac{x_{j_1} - \breve{L}_{j_1}}{\breve{U}_{j_1} - \breve{L}_{j_1}}= \frac{x_{j_2} - \breve{L}_{j_2}}{\breve{U}_{j_2} - \breve{L}_{j_2}} \quad \text{ or }  \quad   \frac{x_{j_1} - \breve{L}_{j_1}}{\breve{U}_{j_1} - \breve{L}_{j_1}}< \frac{x_{j_2} - \breve{L}_{j_2}}{\breve{U}_{j_2} - \breve{L}_{j_2}}.
\end{equation}
In the first case $v''$ is a feasible solution to \eqref{eqn:separation-frac-knap-new-again} that has fewer different components with $v$ and has the same objective value as $v'$. Hence, by repeating this procedure we will eventually  have the second case in which $v''$ is a feasible solution to \eqref{eqn:separation-frac-knap-new-again} that has an objective value strictly smaller than that of  $v'$, which contradicts the optimality of $v'$.

The greedy procedure to solve  \eqref{separationproblem}  and  \eqref{eqn:separation-frac-knap-new-again} can be executed in $\mathcal{O}(n\log(n))$ time through the sorting required to get \eqref{nondecreasingequation}. However, an optimal basic feasible solution $\hat{\alpha}$ to  \eqref{eqn:separation-frac-knap-new-again} can also be obtained in  $\mathcal{O}(n)$ time by solving a weighted median problem (e.g. \cite[Chapter 17.1]{Korte:2000}). This solution can be converted to an optimal solution to \eqref{separationproblem}  in $\mathcal{O}(n)$  time as follows. Because  $\hat{\alpha}$ is a basic feasible solution to \eqref{eqn:separation-frac-knap-new-again}, it has at most one fractional component (see, e.g., \cite[Chapter 3]{Bertsimas:1997}). Take $\hat{I} = \Set{i \in \llbracket n \rrbracket | \hat{\alpha}_i = 1}$. If $\hat{v}$ has one fractional component, take $\hat{h}$ to be this component. Then, because $\hat{\alpha}$ satisfies \eqref{eqn:separation-frac-knap-eq}  we have
\begin{equation}\label{eq:intermdiate-sep-lemma-obj1}
    w_{\hat{h}} (\breve{U}_{\hat{h}} - \breve{L}_{\hat{h}})\hat{\alpha}_{\hat{h}} = \sum_{i=1}^n w_i \breve{U}_i + b -\sum_{i\in \hat{I}} w_i (\breve{U}_i - \breve{L}_i) = \ell(\hat{I})
\end{equation}

Together with $\hat{\alpha}_{\hat{h}}\in (0,1)$, \eqref{eqn:twiddle-one-comp-new} for $I=\hat{I}$ and $h=\hat{h}$, and  \eqref{eq:separation-frac-knap-new-positiveweight} for $i=\hat{h}$, we have $ \ell(\hat{I})>0$ and
\[ \ell(\hat{I}) - \ell(\hat{I} \cup \{\hat{h}\}) >  \ell(\hat{I}).\]
Then $ \ell(\hat{I} \cup \{\hat{h}\}) < 0$ and $(\hat{I},\hat{h})\in \calJ$. Finally, \eqref{eq:intermdiate-sep-lemma-obj1} implies that the objective value of $\hat{\alpha}$ for \eqref{eqn:separation-frac-knap-new-again} is the same as the objective value of $(\hat{I},\hat{h})$ for \eqref{separationproblem}.

If, on the other hand, $\hat{v}$ has no fractional component, then  $\hat{\alpha}$ satisfying  \eqref{eqn:separation-frac-knap-eq}  implies
\begin{equation}\label{eq:intermdiate-sep-lemma-obj2}
0 = \sum_{i=1}^n w_i \breve{U}_i + b -\sum_{i\in \hat{I}} w_i (\breve{U}_i - \breve{L}_i) = \ell(\hat{I}).
\end{equation}
Then, $\ell(\sidx{n})<0$ implies that there exists $\hat{h}\in \sidx{n}\setminus \hat{I}$ such that $\ell(\hat{I}\cup\{\hat{h}\})<0$ and $(\hat{I},\hat{h}) \in \calJ$. Finally, \eqref{eq:intermdiate-sep-lemma-obj2} implies that the objective value of $\hat{\alpha}$ for \eqref{eqn:separation-frac-knap-new-again} is the same as the objective value of $(\hat{I},\hat{h})$ for \eqref{separationproblem}. This conversion of an optimal basic feasible solution for \eqref{eqn:separation-frac-knap-new-again} to a solution to \eqref{separationproblem} also gives an alternate proof to  $\omega(x)\geq  \upsilon(x)$.
\end{proof}

\polylemma*
\begin{proof}
If $w$ and $b$ are rational, then the coefficients of the inequalities in \eqref{eqn:cut-family-2} are also rational numbers with sizes that are polynomial in the sizes of $w$ and $b$. Then the result follows from \autoref{lemma:general_separation} and
\cite[Theorem 7.26]{conforti2014integer}.
\end{proof}

\end{section}

\begin{section}{Propagation algorithms}\label{propagationalgappendix}
\begin{subsection}{Description and analysis of algorithms}

In this section, we provide pseudocode for the propagation-based algorithms described in Section~\ref{propagationsec}. In the scope of a single neuron, Algorithm~\ref{alg:backwards} specifies the framework outlined in Section~\ref{sec:generic_framework} and Algorithm~\ref{alg:iterative} (which requires Algorithm~\ref{alg:forwards}) details our new algorithm proposed in Section~\ref{sec:new_prop_algorithm}. Finally, Algorithm~\ref{alg:full} establishes how to compute bounds for the entire network, considering \DeepPoly{}~\cite{singh2019abstract} and \FastLin{}~\cite{weng2018towards} as possible initial methods.


\begin{algorithm}[H]
\caption{The Backwards Pass for Upper Bounds}
\label{alg:backwards}
\begin{algorithmic}[1]
\Inputs{Input domain $X \subseteq \bbR^{m}$, affine functions $\calL_i(z_{1:i-1})=\sum_{j=1}^{i-1}w^l_{ij}z_j+b^l_i$, $\calU_i(z_{1:i-1})=\sum_{j=1}^{i-1}w^u_{ij}z_j+b^u_i$ for each $i=m+1,\ldots,\eta$, and affine function $\mathcal{C}(z)=\sum_{i=1}^{\eta}c_iz_i + b$}
\Outputs{Upper bound on $\mathcal{C}(z)$, optimal point $x^*\in X$, and boolean vector $(\ubact_{m+1},\ldots,\ubact_\eta)$}\smallskip
\Function{PropagationBound}{$X, \calL, \calU, \calC$}
\State $\ubact_i \gets \texttt{false}$ for all $i = m+1, \ldots, \eta$
\State $Q \gets \Set{i | c_i \neq 0, i > m}$\Comment{Set of variable indices to be substituted}
\State $\texttt{expr} \gets \sum_{i=1}^{\eta} c_i z_i + b$\Comment{Denote by $\texttt{expr.w[i]}$ the coefficient for $z_i$ in $\texttt{expr}$, $\forall i$}
\While{$Q$ is not empty}
\State $i \gets$ pop largest value from $Q$, removing it
\State $\ubact_i \gets (\texttt{expr.w[i]} > 0)$
\State $\texttt{expr} \gets \texttt{expr} - \texttt{expr.w[i]}z_i$\Comment{Remove term from expression}
\If{$\texttt{expr.w[i]} > 0$}
\State $\texttt{expr} \gets \texttt{expr} + \texttt{expr.w[i]} \ \mathcal{U}_i(z_{1:i-1})$
\State $Q \gets Q \cup \Set{j | w_{ij}^u \neq 0, j > m}$
\ElsIf{$\texttt{expr.w[i]} < 0$}
\State $\texttt{expr} \gets \texttt{expr} + \texttt{expr.w[i]} \ \mathcal{L}_i(z_{1:i-1})$
\State $Q \gets Q \cup \Set{j | w_{ij}^l \neq 0, j > m}$
\EndIf
\EndWhile
\State $B, x^* \gets \max_{x\in X} \texttt{expr}$, along with an optimal solution
\State \Return $B, x^*, \ubact$
\EndFunction
\end{algorithmic}
\end{algorithm}


\begin{algorithm}[H]
\caption{The Forward Pass}
\label{alg:forwards}
\begin{algorithmic}[1]
\Inputs{Partial optimal solution $x^*\in X$ from Algorithm~\ref{alg:backwards} and boolean vector $(\ubact_{m+1},\ldots,\ubact_{\eta})$ where $\ubact_i =$ \texttt{true} if the upper bound $\mathcal{U}_i$ was used to substitute variable $i$ in Algorithm 1, or \texttt{false} otherwise.}
\Outputs{Optimal solution $z^*_{1:\eta}$ to~\eqref{eqn:relaxed-problem}}\smallskip
\Function{RecoverCompleteSolution}{$x^*, \ubact$}
\State $z^*_i=x^*_i$ for $i=1,\ldots,m$
\For{$i=m+1,\ldots,\eta$}
\If{$\texttt{ub\_used}_i=\mathtt{true}$}
\State $z^*_i\gets\calU_i(z^*_{1:i-1})$
\Else
\State $z^*_i\gets\calL_i(z^*_{1:i-1})$
\EndIf
\EndFor
\State \Return $z^*_{1:\eta}$
\EndFunction
\end{algorithmic}
\end{algorithm}

\algnewcommand{\IIf}[1]{\State\algorithmicif\ #1\ \algorithmicthen}
\algnewcommand{\IElse}[1]{\unskip\ \algorithmicelse\ #1}
\algnewcommand{\EndIIf}{\unskip\ \algorithmicend\ \algorithmicif}


\algnewcommand{\AlgNote}[1]{%
  \State \textbf{Note:}
  \Statex \hspace*{\algorithmicindent}\parbox[t]{.85\linewidth}{\raggedright #1}
}

\begin{algorithm}[thpb]
\caption{The Iterative Algorithm}
\label{alg:iterative}
\begin{algorithmic}[1]
\Inputs{Input domain $X \subseteq \bbR^m$, initial affine bounding functions $\{\calL_i^{\mathrm{init}}\}_{i=m+1}^\eta$, $\{\calU_i^{\mathrm{init}}\}_{i=m+1}^\eta$, affine function $\mathcal{C} : \bbR^{\eta} \to \bbR$, and number of iterations $T \geq 0$}
\Outputs{An upper bound on $\max_{x \in X} \mathcal{C}(x)$}\smallskip
\Function{TightenedPropagationBound}{$X, \calL, \calU, \calC, k$}
\State $\{\calL_i,\calU_i\}_{i=m+1}^{\eta} \gets \{\calL_i^{\mathrm{init}},\calU_i^{\mathrm{init}}\}_{i=m+1}^{\eta}$
\State $B, x^*, \ubact \gets$ \textsc{PropagationBound}($X, \{\calL_i\}_{i=m+1}^{\eta}, \{\calU_i\}_{i=m+1}^{\eta}, \mathcal{C}$)
\For{$\texttt{iter} = 1, \ldots, T$}
    \State $z^*_{1:\eta} \gets $ \textsc{RecoverCompleteSolution}($x^*, \ubact$)
    \For{$i = m + 1, \ldots, \eta$}
        \State $\calU_i', v \gets$ most violated inequality w.r.t.\  $z^*_{1:\eta}$ from \eqref{eqn:cut-family-2} (per Prop.~\ref{lemma:general_separation}) and its violation
        \IIf{$v > 0$} $\calU_i \gets \calU_i'$ \EndIIf
    \EndFor
\State $B', x^*, \ubact \gets$ \textsc{PropagationBound}($X, \{\calL_i\}_{i=m+1}^{\eta}, \{\calU_i\}_{i=m+1}^{\eta}, \mathcal{C}$)
\IIf{$B' < B$} $B \gets B'$ \EndIIf
\EndFor
\State \Return $B$
\EndFunction
\end{algorithmic}
\end{algorithm}


\begin{algorithm}[thpb]
\caption{\FastCtV{} Algorithm}
\label{alg:full}
\begin{algorithmic}[1]
\Inputs{A feedforward neural network as defined in~\eqref{eqn:topo-network-equations} (with input domain $X$, ReLU neurons $i = m+1, \ldots, N$, and a single affine output neuron indexed by $N+1$),  \texttt{initial\_method} $\in \{\DeepPoly{}, \FastLin{}\}$, and number of iterations per neuron $T \geq 0$ (note that if $T = 0$, we recover \DeepPoly{} or \FastLin{})}
\Outputs{Lower and upper bounds $\{\hat{L}_i, \hat{U}_i\}_{i=1}^{N+1}$ on the pre-activation function (if ReLU) or output (if affine) of neuron $i$}
\smallskip
\Function{FastC2V}{$X, W, b, \mathtt{initial\_method}, T$}
\For{$i = m+1, \ldots, N+1$}
\State $\calC(z_{1:i-1}) \gets \sum_{j=1}^{i-1} w_{i,j} z_j + b_i$
\State $\hat{L}_i \gets -$ \textsc{TightenedPropagationBound($X, \{\calL_j\}_{j=m+1}^{i-1}, \{\calU_j\}_{j=m+1}^{i-1}, -\calC, T$)}
\State $\hat{U}_i \gets$ \textsc{TightenedPropagationBound($X, \{\calL_j\}_{j=m+1}^{i-1}, \{\calU_j\}_{j=m+1}^{i-1}, \calC, T$)}

\IIf{$i = N+1$} \textbf{break} \EndIIf

\State{$\triangleright$ Build bounding functions $\calL_i$ and $\calU_i$ for subsequent iterations}

\If{$\hat{L}_i \geq 0$} \Comment{ReLU $i$ is always active for any $z_{1:m} \in X$}
    \State $\calL_i(z_{1:i-1}) \gets \sum_{j=1}^{i-1} w_{i,j} z_j + b_i$
    \State $\calU_i(z_{1:i-1}) \gets \sum_{j=1}^{i-1} w_{i,j} z_j + b_i$
\ElsIf{$\hat{U}_i \leq 0$} \Comment{ReLU $i$ is always inactive for any $z_{1:m} \in X$}
    \State $\calL_i(z_{1:i-1}) \gets 0$
    \State $\calU_i(z_{1:i-1}) \gets 0$
\Else
    \State $\calU_i(z_{1:i-1}) \gets \frac{\hat{U}_i}{\hat{U}_i - \hat{L}_i} (\sum_{j=1}^{i-1} w_{i,j} z_j + b_i - \hat{L}_i)$
    \If{\texttt{initial\_method} $=$ \DeepPoly{}}
    \IIf{$|\hat{L}_i| \geq |\hat{U}_i|$} $\calL_i(z_{1:i-1}) \gets 0$ \IElse $\calL_i(z_{1:i-1}) \gets \sum_{j=1}^{i-1} w_{i,j} z_j + b_i$
    \EndIIf
    \Else \quad $\triangleright$ \texttt{initial\_method} $=$ \FastLin{}
    \State $\calL_i(z_{1:i-1}) \gets \frac{\hat{U}_i}{\hat{U}_i - \hat{L}_i} (\sum_{j=1}^{i-1} w_{i,j} z_j + b_i)$
    \EndIf
\EndIf
\EndFor
\State \Return $\{\hat{L}_i, \hat{U}_i\}_{i=1}^{N+1}$
\EndFunction
\end{algorithmic}
\end{algorithm}

\begin{restatable}{proposition}{forwards} \label{prop:forwards-solves}
The solution $z^*$ returned by Algorithm~\ref{alg:forwards} is optimal for the relaxed problem \eqref{eqn:relaxed-problem}.
\end{restatable}
\begin{proof}
Denote by $\texttt{expr}^{k} \defeq \sum_{j \in J^k} \texttt{w}_j^k z_j + \texttt{b}^k$ the expression $\texttt{expr}$ at the end of iteration $k = 1, \ldots, K$ of the while loop in Algorithm~\ref{alg:backwards}, for some subsets $J^1,\ldots,J^K \subseteq \llbracket \eta \rrbracket$, and let $\texttt{expr}^0$ be the initial \texttt{expr} as defined in line 5, i.e.\ $\mathcal{C}(z)$. For each $k = 0,\ldots,K-1$, we obtain $\texttt{expr}^{k+1}$ by replacing, for some $i$, $z_i$ by $\mathcal{U}_i(z_{1:i-1})$ if $\texttt{w}_i^k > 0$, or by  $\mathcal{L}_i(z_{1:i-1})$ if if $\texttt{w}_i^k < 0$. Note that if $\texttt{w}_i^k = 0$, we can safely ignore any substitution because it will not affect the expression. Due to the constraints~\eqref{eqn:relaxedConstr}, this substitution implies that $\texttt{expr}^k \leq \texttt{expr}^{k+1}$ for any $z_{1:m} \in X$. This inductively establishes that, restricting to $z_{1:m} \in X$,
\begin{align}\label{eqn:inductionConc}
\mathcal{C}(z) = \sum_{j=1}^{\eta} c_j z_j + b \leq \sum_{j \in J^1} \texttt{w}_j^1 z_j + \texttt{b}^1 \leq \ldots \leq \sum_{j \in J^K} \texttt{w}_j^K z_j + \texttt{b}^K,
\end{align}
Note that $J^K \subseteq \{z_1, \ldots, z_m\}$ since we have made all the substitutions possible for $i > m$. Therefore, the optimal value of~\eqref{eqn:relaxed-problem} is upper-bounded by the bound corresponding to the solution returned by Algorithm~\ref{alg:backwards}, that is,
\begin{align*}
    \max \Set{\mathcal{C}(z) | z_{1:m} \in X,\ \eqref{eqn:relaxedConstr}} \leq \max \Set{\textstyle\sum_{j \in J^K} \texttt{w}_j^K z_j + \texttt{b}^K | z_{1:m} \in X}.
\end{align*}

To see that this upper bound is achieved, observe that each inequality in~\eqref{eqn:inductionConc} holds as equality if we substitute $z_j=z^*_j$ for all $j$, by construction of Algorithm~\ref{alg:forwards} and boolean vector $(\ubact_{m+1},\ldots,\ubact_{\eta})$. Note also that $z^*$ satisfies~\eqref{eqn:relaxedConstr} by construction. That is, we have a feasible $z^*$ such that $\mathcal{C}(z^*)$ is no less than the optimal value of~\eqref{eqn:relaxed-problem}, and thus $z^*$ must be an optimal solution.
\end{proof}

We would like to highlight to the interested reader that this result can also be derived from an argument using Fourier-Motzkin elimination~\cite[Chapter 2.8]{Bertsimas:1997} to project out the intermediate variables $z_{m+1:\eta}$. Notably, as each inequality neuron has exactly one inequality upper bounding and one inequality lower bounding its post-activation value, this projection does not produce an ``explosion'' of new inequalities as is typically observed when applying Fourier-Motzkin to an arbitrary polyhedron.

Define $C \defeq |\Set{i \in \llbracket \eta \rrbracket | c_i \neq 0}|$ and suppose that we use the affine bounding inequalities from \FastLin{} or \DeepPoly{}. Let $T$ be the number of iterations in Algorithm~\ref{alg:iterative}, $\mathit{opt}(X)$ be the time required to maximize an arbitrary affine function over $X$, and $A$ be the number of arcs in the network (i.e.\ nonzero weights).

\begin{observation}\label{obs:algorithm_single_complexity}
    Algorithm~\ref{alg:backwards} runs in $\mathit{opt}(X) + \mathcal{O}(C + A)$ time. Algorithm~\ref{alg:forwards} runs in $\mathcal{O}(A)$ time. Algorithm~\ref{alg:iterative} runs in $(T+1) \mathit{opt}(X) + \mathcal{O}(T(C+A))$ time.
\end{observation}

\begin{observation}\label{obs:algorithm_all_complexity}
Algorithm~\ref{alg:full} takes  $\mathcal{O}(NT (\mathit{opt}(X)+A))$ time if $T \geq 1$. If $T = 0$, then Algorithm~\ref{alg:full} takes $\mathcal{O}(N(\mathit{opt}(X)+A))$ time.
\end{observation}

\end{subsection}
\begin{subsection}{Proofs of other results from \autoref{propagationsec}}
\relaxequallp*
\begin{proof}
For any $x\in X$, by definition of validity in~\eqref{eqn:affineValid}, setting $z_i \gets z_i(x)$ for all $i=1,\ldots,N$ yields a feasible solution to~\eqref{eqn:relaxed-problem} with objective value $c\rbra{z_{1:N}\rbra{x}}$, completing the proof.
\end{proof}

\end{subsection}
\end{section}

\begin{section}{An example for \FastCtV{}}\label{app:fastc2v_example}

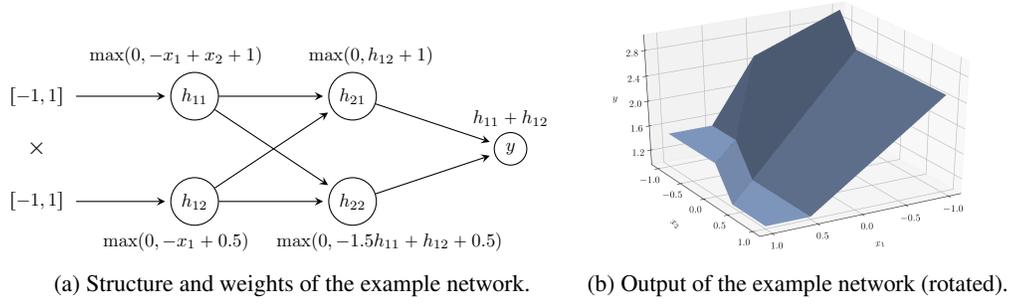
\begin{figure}[t]
\centering
\begin{subfigure}[t]{.55\linewidth}
\begin{tikzpicture}[scale=0.7, every node/.style={scale=0.7}]]
\node at (0,0) {$[-1,1]$};
\node at (0,1) {\Large$\times$};
\node at (0,2) {$[-1,1]$};
\node[draw=black, circle, label={[xshift=-1em]above:$\max(0, -x_1 + x_2 + 1)$}](v1) at (3,2) {$h_{11}$};
\node[draw=black, circle, label={[xshift=-1em]below:$\max(0, -x_1 + 0.5)$}](v2) at (3,0) {$h_{12}$};
\draw[->, >=stealth, shorten >= 2pt]  (0.75,2) -- (v1);
\draw[->, >=stealth, shorten >= 2pt]  (0.75,0) -- (v2);
\node[draw=black, circle, label={[xshift=1em]above:$\max(0, h_{12} + 1)$}](v3) at (6,2) {$h_{21}$};
\node[draw=black, circle, label={[xshift=2em]below:$\max(0, -1.5 h_{11} + h_{12} + 0.5)$}](v4) at (6,0) {$h_{22}$};
\draw[->, >=stealth, shorten >= 2pt]  (v1) -- (v3);
\draw[->, >=stealth, shorten >= 2pt]  (v1) -- (v4);
\draw[->, >=stealth, shorten >= 2pt]  (v2) -- (v3);
\draw[->, >=stealth, shorten >= 2pt]  (v2) -- (v4);
\node[draw=black, circle, label={above:$h_{11} + h_{12}$}](v5) at (9,1) {$y$};
\draw[->, >=stealth, shorten >= 2pt]  (v3) -- (v5);
\draw[->, >=stealth, shorten >= 2pt]  (v4) -- (v5);
\end{tikzpicture}\caption{Structure and weights of the example network.}
\end{subfigure}
\begin{subfigure}[t]{.4\linewidth}
\centering
\includegraphics[width=\textwidth]{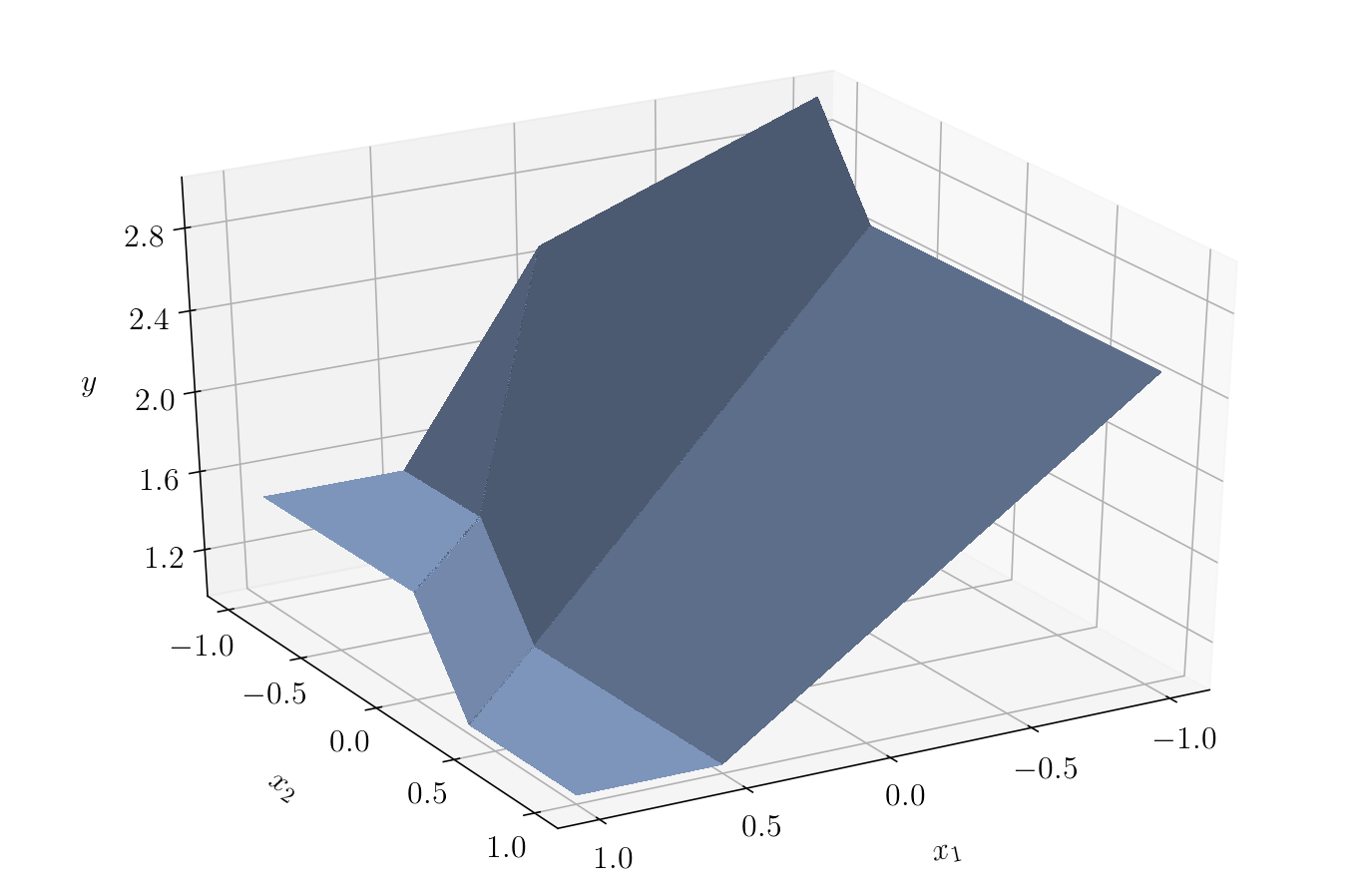}\caption{Output of the example network (rotated).}
\end{subfigure}
\caption{An example network with 4 ReLUs on which we simulate \FastCtV{}.}\label{fig:fastc2v_example_network}
\end{figure}

In this section, we walk through the \FastCtV{} algorithm step-by-step for the following $\mathbb{R}^2 \to \mathbb{R}$ network with four ReLUs, also illustrated in Figure~\ref{fig:fastc2v_example_network}:
\begin{align*}
    x_1 &\in [-1, 1]\\
    x_2 &\in [-1, 1]\\
    h_{11} &= \max(0, -x_1 + x_2 + 1)\\
    h_{12} &= \max(0, -x_1 + 0.5)\\
    h_{21} &= \max(0, h_{12} + 1)\\
    h_{22} &= \max(0, -1.5 h_{11} + h_{12} + 0.5)\\
    y &= h_{21} + h_{22}
\end{align*}

Our goal is to compute an upper bound for $y$ using \FastCtV.

Our procedure requires lower and upper bounds for each pre-activation function, and this can be obtained by running the same algorithm for each neuron, layer by layer. For simplicity, in this example we start from bounds computed via interval arithmetic.

Denote by $\hat{h}_{ij}$ the pre-activation function of $h_{ij}$. To apply interval arithmetic, we simply substitute the variables by their lower or upper bounds as to minimize or maximize them (applying the ReLU activation function when needed). For example, the interval arithmetic upper bound of $\hat{h}_{11}$ is $-1.5\times(-1) + 1 + 0.5 = 3$. Starting at $x_1 \in [-1,1]$ and $x_2 \in [-1, 1]$, we have:
\begin{align*}
    \hat{h}_{11} &\in [-1, 3] \hspace{2.55em} (h_{11} \in [0, 3])& \hat{h}_{21} &\in [1, 2.5] \quad (h_{21} \in [1, 2.5])\\
    \hat{h}_{12} &\in [-0.5, 1.5] \quad (h_{12} \in [0, 1.5])& \hat{h}_{22} &\in [-4, 2] \quad (h_{22} \in [0, 2])
\end{align*}

Note that $\hat{h}_{21} \geq 0$ for any input $x \in [-1,1]^2$, and thus we may infer that the ReLU will always be active. That is, we can assume $h_{21} = h_{12} + 1$. This linearization step not only can have a large impact in bound strength, but also is required for correctness, as the formulations assume that the lower bound is negative and the upper bound positive.

Therefore, we drop $h_{21}$ altogether and set
\begin{align*}
    y = h_{12} + h_{22} + 1.
\end{align*}

Observe that interval arithmetic already gives us a simple upper bound on $y$ of 4.5.

We begin by applying \texttt{DeepPoly}~\cite{singh2019abstract} (or \texttt{CROWN-Ada}~\cite{zhang2018efficient}), following Algorithm~\ref{alg:backwards}. Consider briefly a ReLU $y = \max(0, w^\top x + b)$ with pre-activation bounds $[\hat{L}, \hat{U}]$. In \texttt{DeepPoly}, we select the lower bounding inequality to be $y \geq 0$ if $|\hat{L}| \geq |\hat{U}|$, or $y \geq w^\top x + b$ otherwise. The upper bounding inequality comes from the $\Delta$-relaxation and can be expressed as $y \leq \frac{\hat{U}}{\hat{U}-\hat{L}} (w^\top x + b - \hat{L})$. Thus, based on the previously computed bounds, we have:
\begin{align*}
    -x_1 + x_2 + 1 &\leq h_{11} \leq -\frac{3}{4} x_1 + \frac{3}{4} x_2 + \frac{3}{2}\\
    -x_1 + \frac{1}{2} &\leq h_{12} \leq -\frac{3}{4} x_1 + \frac{3}{4}\\
    0 &\leq h_{22} \leq -\frac{1}{2} h_{11} + \frac{1}{3} h_{12} + \frac{3}{2}
\end{align*}

The next step is to maximize $y$ over the relaxation given by the above inequalities plus bounds (including on the input). We replace variables with the above bounding inequalities, layer by layer. Since we are maximizing, we use the upper bounding inequality if the corresponding coefficient is positive, or the lower bound inequality otherwise. This maintains validity of the inequality throughout the process.
\begin{align*}
    y = h_{12} + h_{22} + 1 &\leq h_{12} + \left(-\frac{1}{2} h_{11} + \frac{1}{3} h_{12} + \frac{3}{2}\right) + 1\\
    &= -\frac{1}{2} h_{11} + \frac{4}{3} h_{12} + \frac{5}{2}\\
    &\leq -\frac{1}{2}\left(-x_1 + x_2 + 1\right) + \frac{4}{3}\left(-\frac{3}{4} x_1 + \frac{3}{4}\right) + \frac{5}{2}\\
    &= -\frac{1}{2} x_1 - \frac{1}{2} x_2 + 3
\end{align*}

Now that we have inferred the above upper bounding inequality on $y$, we convert it into an upper bound by solving the simple problem $\max_{x \in [-1,1]^2} -\frac{1}{2} x_1 - \frac{1}{2} x_2 + 3$, which yields $4$, with an optimal solution $(-1, -1)$. This is the resulting upper bound from the \texttt{DeepPoly} algorithm.

We next show how to tighten it with \FastCtV. The first step is to recover an actual optimal solution of the relaxation above. This is the forward pass described in Algorithm~\ref{alg:forwards}.

We first make note that we used the upper bounding inequality for $h_{12}$ and $h_{22}$ and the lower bounding inequality for $h_{11}$. We start from the optimal solution in the input space, $(-1, -1)$, and recover values for each $h_{ij}$ and $y$ according to the bounding inequalities used, considering them to be equalities. For example, $h_{11} = -(-1) + (-1) + 1 = 1$. The result is the solution $\bar{p} = (-1, -1, 1, 1.5, 1.5, 4)$ in the $(x_1, x_2, h_{11}, h_{12}, h_{22}, y)$-space.

We now perform the main step of \FastCtV, which is to swap upper bounding inequalities based on $p$. They are swapped to whichever inequality is violated by $p$, or not swapped if no inequality is violated for a given ReLU neuron.

In this example, we skip $h_{11}$ and $h_{12}$ for simplicity as no swapping occurs, and we focus on $h_{22}$.
For $h_{22}$, the relevant values of $\bar{p}$ are $\bar{h}_{11} = 1$, $\bar{h}_{12} = 1.5$, and $\bar{h}_{22} = 1.5$. Normally, we would solve the separation problem at this point, but for illustrative purposes we list out all possible upper bounding inequalities that we can swap to.

Recall Theorem~\ref{thm:convex-hull} and compute:
\begin{align*}
    \ell(\varnothing) &= 2 & \ell(\{1\}) &= -2.5\\
    \ell(\{2\}) &= 0.5 & \ell(\{1,2\}) &= -4
\end{align*}

Based on these values, we have $\mathcal{J} = \{(\varnothing, 1), (\{2\}, 1)\}$, or in other words, two possible inequalities to swap to. By following the formulation in Theorem~\ref{thm:convex-hull}, we obtain the inequalities
\begin{align}
    h_{22} &\leq -\frac{2}{3} h_{11} + 2\label{eq:example_ineq_1}\\
    h_{22} &\leq -\frac{1}{6} h_{11} + h_{12} + \frac{1}{2}\label{eq:example_ineq_2}
\end{align}

\begin{figure}[t]
\centering
\begin{subfigure}[t]{.3\linewidth}
\centering
\includegraphics[width=\textwidth]{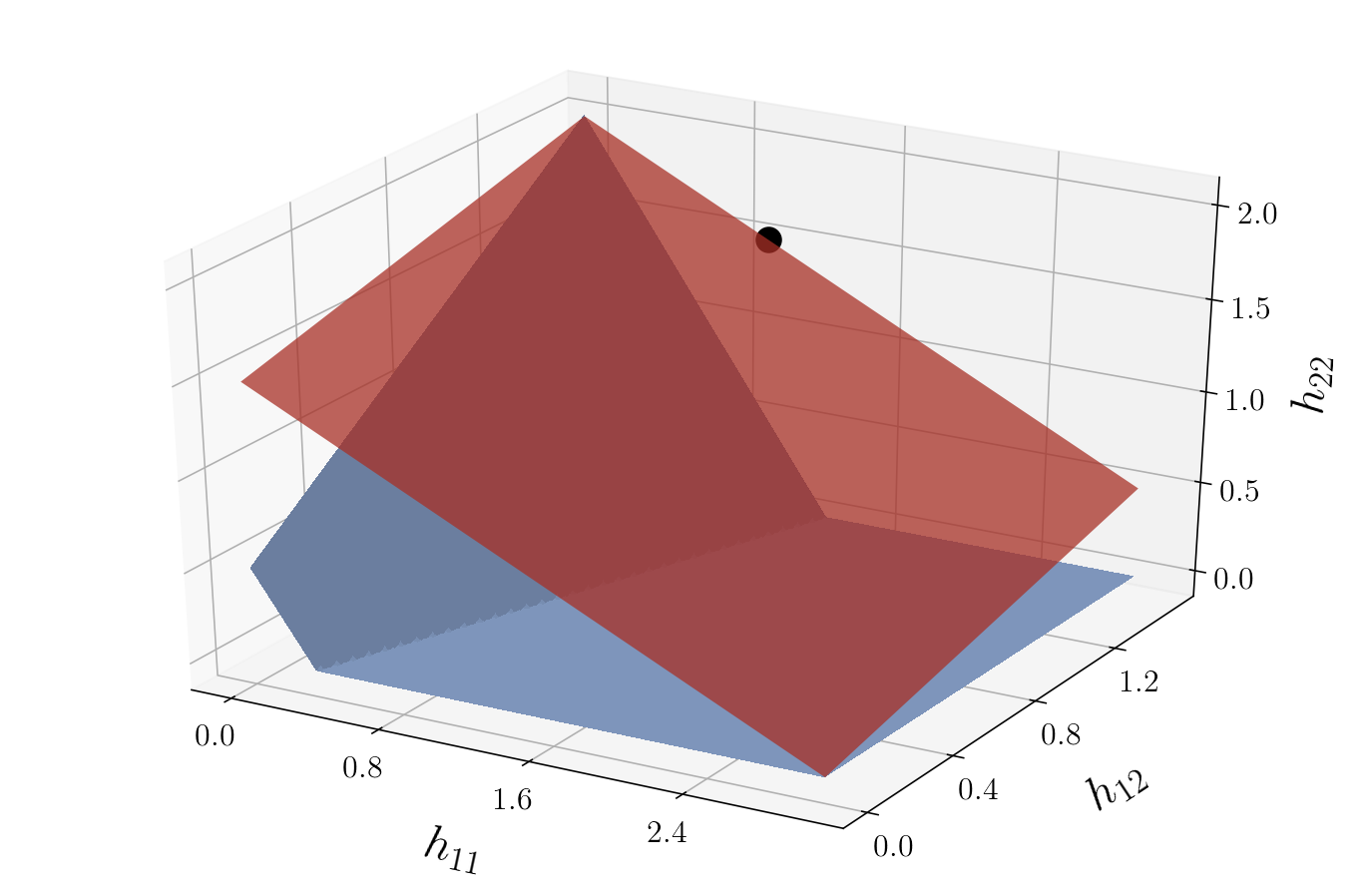}\caption{Original inequality from the $\Delta$-relaxation.}
\end{subfigure}
\begin{subfigure}[t]{.3\linewidth}
\centering
\includegraphics[width=\textwidth]{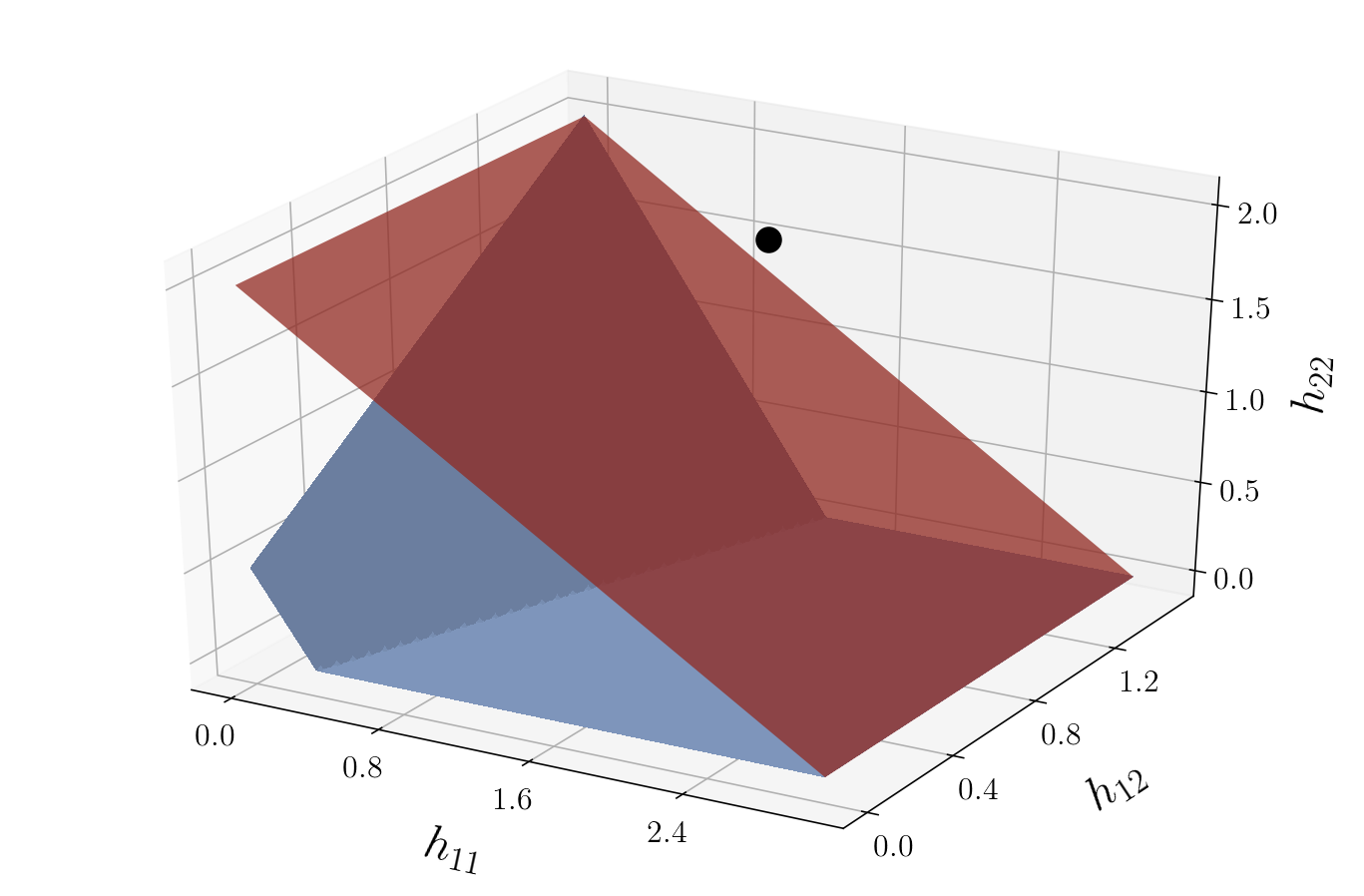}\caption{Inequality~\eqref{eq:example_ineq_1}.}
\end{subfigure}
\begin{subfigure}[t]{.3\linewidth}
\centering
\includegraphics[width=\textwidth]{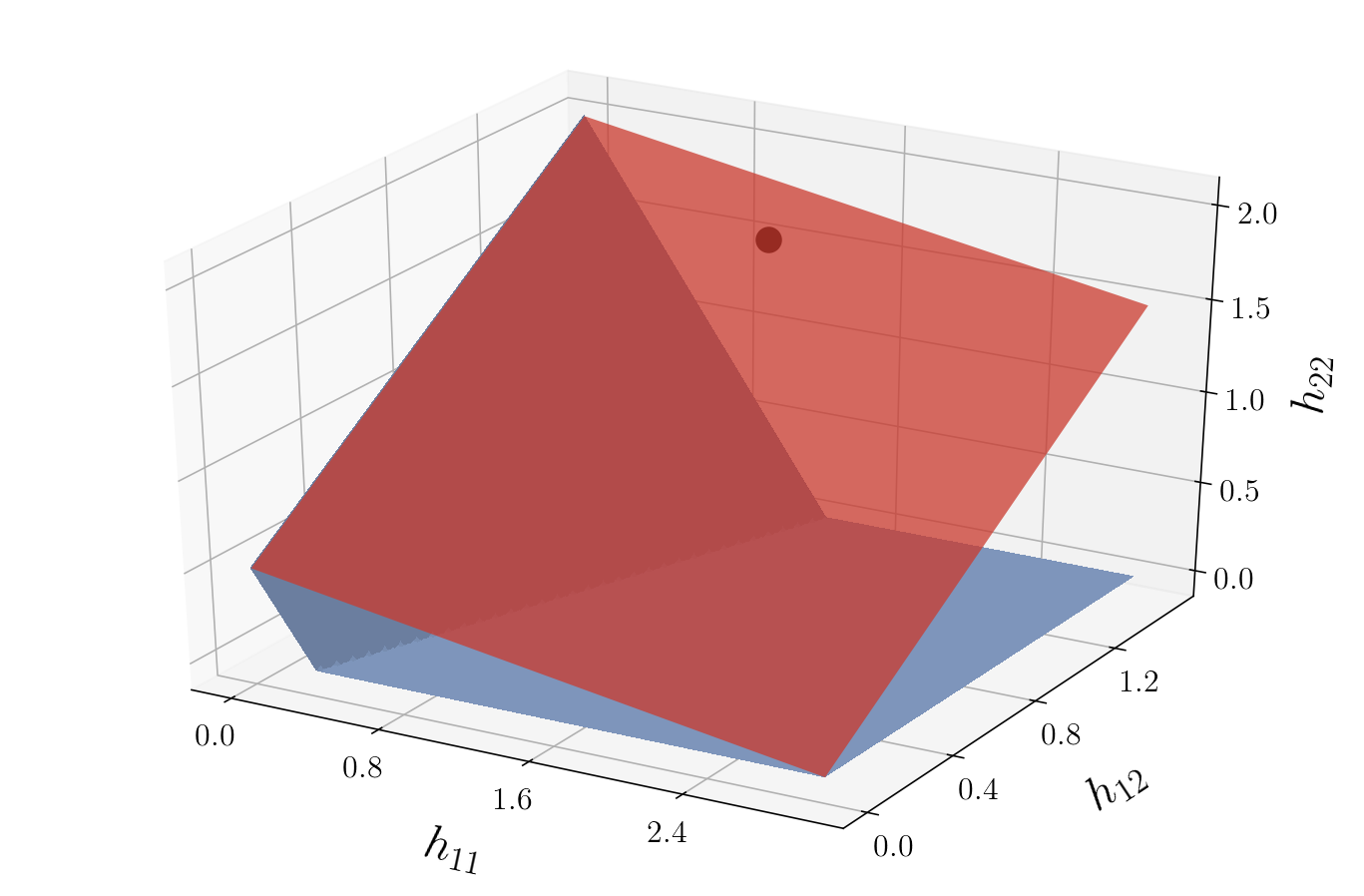}\caption{Inequality~\eqref{eq:example_ineq_2}.}
\end{subfigure}
\caption{Three options of upper bounding inequalities for $h_{22}$. The black point depicts the solution that we would like to separate, which is cut off by inequality~\eqref{eq:example_ineq_1}.}\label{fig:fastc2v_example_inequalities}
\end{figure}

These inequalities are illustrated in Figure~\ref{fig:fastc2v_example_inequalities}. We observe that our point $p$ is cut off by the inequality~\eqref{eq:example_ineq_1}: $1.5 = \bar{h}_{22} > -\frac{2}{3} \bar{h}_{11} + 2 = \frac{4}{3} \approx 1.333$. Therefore, for this neuron, we swap the upper bounding inequality to~\eqref{eq:example_ineq_1}. In other words, our pair of inequalities for $h_{22}$ is now:
\begin{align*}
    0 &\leq h_{22} \leq -\frac{2}{3} h_{11} + 2
\end{align*}

The last step of \FastCtV\ is to redo the backward propagation with the swapped inequalities and recompute the bound. We obtain:
\begin{align*}
    y = h_{12} + h_{22} + 1 &\leq h_{12} + \left(-\frac{2}{3} h_{11} + 2\right) + 1\\
    &= -\frac{2}{3} h_{11} + h_{12} + 3\\
    &\leq -\frac{2}{3}\left(-x_1 + x_2 + 1\right) + \left(-\frac{3}{4} x_1 + \frac{3}{4}\right) + 3\\
    &= -\frac{1}{12} x_1 - \frac{2}{3} x_2 + \frac{37}{12}\\
\end{align*}

Solving $\max_{x \in [-1,1]^2} -\frac{1}{12} x_1 - \frac{2}{3} x_2 + \frac{37}{12}$ gives us an improved bound of $\frac{23}{6} \approx 3.833$, completing the \FastCtV\ algorithm for upper bounding $y$. Note that this procedure is not guaranteed to improve the initial bound, and in general we take the best between the initial bound and the new one.

Incidentally, we observe in Figure~\ref{fig:fastc2v_example_inequalities}(a) that the big-$M$ inequality from the $\Delta$-relaxation is, in general, not facet-defining for the convex hull of the feasible points depicted in blue. This explains why it can not be directly reconstructed from our convex hull description~\eqref{eqn:relu-conv-hull}.

\end{section}

\begin{section}{Implementation details}\label{app:implementation_details}

In this section, we add to the implementation details provided in  Section~\ref{sec:computational}.

The implementation of the propagation-based algorithm involves the following details:

\begin{itemize}
    \item It may occur that the result of Algorithm~\ref{alg:backwards} has zero coefficients for some variables $x_i$, in which case any feasible value for $x_i$ produces an optimal solution. For those variables, we select the midpoint between the lower bound and upper bound to proceed with Algorithm~\ref{alg:forwards}.
    \item We find that running more than one iteration of the propagation-based algorithm does not yield improving results. A possible reason for this is that while these inequalities are stronger in some portions of the input space, they are looser by themselves in others, and balancing this can be difficult. Improving this trade-off however is outside the scope of this paper.
    \item We use no tolerance on violation. That is, every violated inequality is swapped in.

\end{itemize}

The implementation of the LP-based algorithm involves the following details:

\begin{itemize}
    \item We find that the Conv networks examined are very numerically unstable for LPs due to the presence of very small weights in the networks. Taking no action results in imprecise solutions, sometimes resulting in infeasible LPs being constructed. To improve on this instability, we consider as zero any weight or generated bound below $10^{-5}$. In addition, we run \DeepPoly{} before the LP to quickly check if the neuron can be linearized. This is applied only to the LP-based methods. Note that the default feasibility and optimality tolerances in Gurobi are $10^{-6}$. With this, we end up solving an approximate problem rather than the exact problem, though arguably it is too difficult to solve these numerically unstable LPs with high precision and reasonable time in practice.
    \item For separation, we implement the $O(n \log n)$ version of the algorithm based on sorting instead of the $O(n)$ version.
    \item For each bound computed, we generate new cuts from scratch. More specifically, when solving for each bound, we make a copy of the model and its LP basis from the previous solve, run the LP solve and cut loop, retrieve the bound, and then discard this copy of the model.
    \item We add cuts whose violation exceeds a tolerance of $10^{-5}$.
    \item  In the context of mixed-integer programming, it is well known that selecting a smaller subset of cuts to add can be very beneficial to reduce solving time, but for simplicity, we perform no cut selection in this method.
    \item An alternative to the LP-based method is to solve a MIP with analogous cutting planes with binary variables~\cite{anderson2020strong}, but we find that this method, free of binary variables, is more lightweight and effective even without cut selection and all the presolve functionalities of modern MIP solvers. The ability to solve these LPs very quickly is important since we solve them at every neuron. In addition, this gives us more fine-grained control on the cuts, providing a better opportunity to evaluate our inequalities.
\end{itemize}

The implementation of all algorithms involve the following details:

\begin{itemize}
    \item We attempt to linearize each neuron with simple interval arithmetic before running a more expensive procedure. This makes a particularly large difference in solving time for the Conv networks, in which many neurons are linearizable.
    \item As done in other algorithms in the literature, we elide the last affine layer, a step that is naturally incorporated in the framework from Section~\ref{sec:generic_framework}. In other words, we do not consider the last affine layer to be a neuron but to be the objective function.
    \item We fully compute the bounds of all neurons in the network, including differences of logits. We make no attempt to stop early even if we have the opportunity to infer robustness earlier.
    \item When solving the verification problem, scalar bounds on the intermediate neurons only need to be computed once per input image (i.e. once per set $X$), and can be reused for each target class (i.e. reused for different objectives $c$).
\end{itemize}

The details of the networks from the ERAN dataset~\cite{eran_benchmark} are the following. To simplify notation, we denote a dense layer by \texttt{Dense(size, activation)} and a convolutional layer by \texttt{Conv2D(number of filters, kernel size, strides, padding, activation)}.

\begin{itemize}
    \item 6x100: $5 \times$ \texttt{Dense(100, ReLU)} followed by \texttt{Dense(10, ReLU)}. This totals 510 units. Trained on the MNIST dataset with no adversarial training.
    \item 9x100: $8 \times$ \texttt{Dense(100, ReLU)} followed by \texttt{Dense(10, ReLU)}. This totals 810 units. Trained on the MNIST dataset with no adversarial training.
    \item 6x200: $5 \times$ \texttt{Dense(200, ReLU)} followed by \texttt{Dense(10, ReLU)}. This totals 1010 units. Trained on the MNIST dataset with no adversarial training.
    \item 6x200: $8 \times$ \texttt{Dense(200, ReLU)} followed by \texttt{Dense(10, ReLU)}. This totals 1610 units. Trained on the MNIST dataset with no adversarial training.
    \item MNIST ConvSmall: \texttt{Conv2D(16, (4,4), (2,2), valid, ReLU)}, \texttt{Conv2D(32, (4,4), (2,2), valid, ReLU)}, \texttt{Dense(100, ReLU)}, \texttt{Dense(10, linear)}. This totals 3604 units. Trained on the MNIST dataset with no adversarial training.
    \item MNIST ConvBig: \texttt{Conv2D(32, (3,3), (1,1), same, ReLU)}, \texttt{Conv2D(32, (4,4), (2,2), same, ReLU)}, \texttt{Conv2D(64, (3,3), (1,1), same, ReLU)}, \texttt{Conv2D(64, (4,4), (2,2), same, ReLU)}, \texttt{Dense(512, ReLU)}, \texttt{Dense(512, ReLU)}, \texttt{Dense(10, linear)}. This totals 48064 units. Trained on the MNIST dataset with DiffAI for adversarial training.
    \item CIFAR-10 ConvSmall: \texttt{Conv2D(16, (4,4), (2,2), valid, ReLU)}, \texttt{Conv2D(32, (4,4), (2,2), valid, ReLU)}, \texttt{Dense(100, ReLU)}, \texttt{Dense(10, linear)}. This totals 4852 units. Trained on the CIFAR-10 dataset with projected gradient descent for adversarial training.
\end{itemize}

\end{section}

\begin{section}{Supplementary computational results}\label{app:supplementary}

\begin{figure}[t]
\centering
\includegraphics[width=0.95\textwidth]{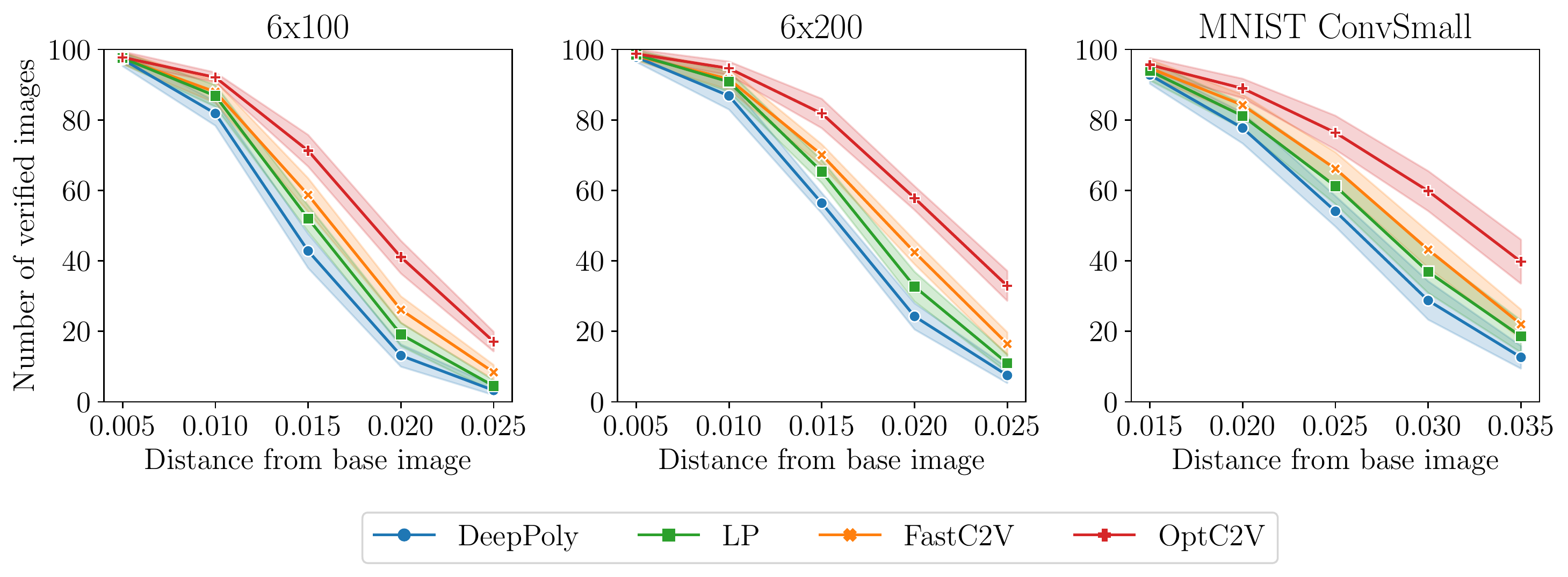}\caption{Number of verified images by each method given various values of the allowed distance from the base image. Lines are averages over 16 randomly initialized networks and error bands represent standard deviation.}\label{fig:supp_variance}
\end{figure}

\begin{figure}[htbp]
\centering
\includegraphics[width=0.9\textwidth]{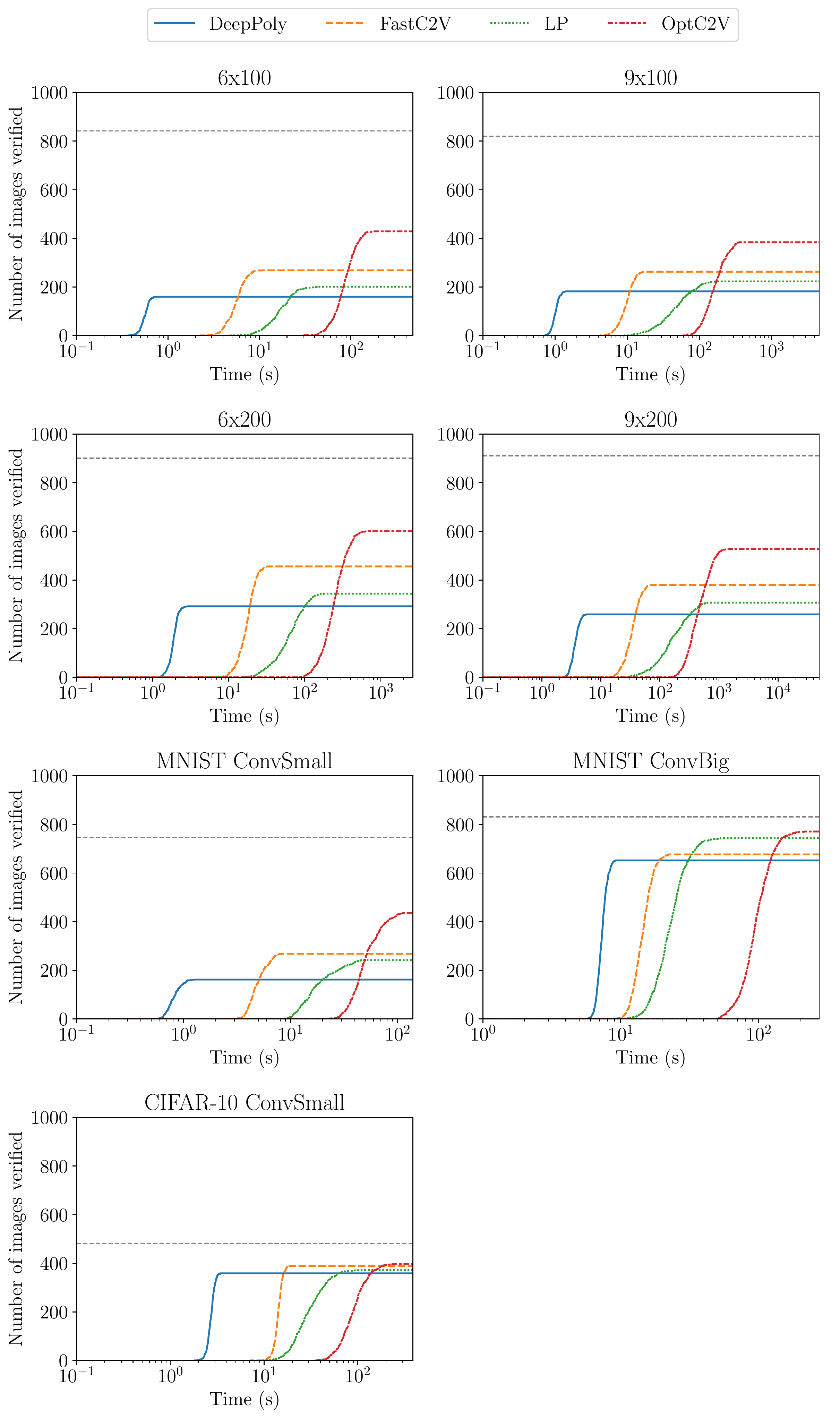}\caption{Survival plots for the results in Section~\ref{sec:computational}. The horizontal dashed line is the upper bound on the number of verifiable images.}\label{fig:supp_survival}
\end{figure}

We computationally examine the sensitivity of the algorithms in this paper to different training initializations and distances from the base image.




We focus on networks for the MNIST dataset. The first two architectures 6x100 and 6x200 have 6 hidden layers of 100 and 200 ReLUs respectively, followed by a linear output layer of 10 ReLUs (this differs slightly from the ERAN networks of the same name described in Appendix~\ref{app:implementation_details}). The MNIST ConvSmall architecture is the same as described in Appendix~\ref{app:implementation_details}. Average test accuracies are 97.04\%, 97.61\%, and 98.63\% respectively. For each architecture and distance, we train 16 randomly initialized networks. Each network is trained with a learning rate of 0.001 for 10 epochs using the Adam training algorithm, without any adversarial training.

Figure~\ref{fig:supp_variance} illustrates the average number of verified images. The error bands represent standard deviation over the 16 networks. We observe that OptC2V and FastC2V perform well across different networks and distances.

In addition, Figure~\ref{fig:supp_survival} depicts survival plots for the results from Table~\ref{tab:computational_results}: the number of images that can be verified given individual time budgets.

\end{section}

\end{document}